\documentclass[sigconf]{acmart}

\author{Ren Kishimoto}
\authornote{Equal contribution.}
\affiliation{
  \institution{Institute of Science Tokyo}
  \city{Tokyo}
  \country{Japan}
}
\email{kishimoto.r.ab@m.titech.ac.jp}

\author{Koichi Tanaka}
\authornotemark[1]
\affiliation{
  \institution{Keio University}
  \city{Tokyo}
  \country{Japan}
}
\email{kouichi\_1207@keio.jp}

\author{Haruka Kiyohara}
\affiliation{
  \institution{Cornell University}
  \state{Ithaca}
  \country{USA}
}
\email{hk844@cornell.edu}

\author{Yusuke Narita}
\affiliation{
  \institution{Yale University}
  \city{New Haven, CT}
  \country{USA}
}
\email{yusuke.narita@yale.edu}

\author{Yasuo Yamamoto}
\affiliation{
  \institution{LY Corporation}
  \city{Tokyo}
  \country{Japan}
}
\email{yasyamam@lycorp.co.jp}

\author{Nobuyuki Shimizu}
\affiliation{
  \institution{LY Corporation}
  \city{Tokyo}
  \country{Japan}
}
\email{nobushim@lycorp.co.jp}

\author{Yuta Saito}
\affiliation{
  \institution{Hanjuku-kaso, Co., Ltd.}
  \city{Tokyo}
  \country{Japan}
}
\email{saito@hanjuku-kaso.com}

\renewcommand{\shortauthors}{Ren Kishimoto et al.}

\usepackage{amsthm}
\usepackage{amsmath}
\usepackage{mathtools}
\usepackage{graphicx}
\usepackage{subcaption}
\usepackage{algorithm}
\usepackage{algorithmic}
\usepackage{tcolorbox}

\theoremstyle{plain}
\newtheorem{condition}{Condition}[section]
\newtheorem{theorem}{Theorem}[section]
\newtheorem{corollary}{Corollary}[section]
\newtheorem{proposition}{Proposition}[section]

\DeclareMathOperator*{\argmax}{arg\,max}

\newcommand{\mE}{\mathbb{E}}
\newcommand{\mV}{\mathbb{V}}

\newcommand{\calD}{\mathcal{D}}
\newcommand{\calX}{\mathcal{X}}
\newcommand{\calA}{\mathcal{A}}

\newcommand{\nabt}{\nabla_{\theta}}
\newcommand{\score}{s_{\theta}}
\newcommand{\firstp}{\pi_{\theta}^{1st}}
\newcommand{\secondp}{\pi_{\phi}^{2nd}}

\newcommand{\overallpolicy}{\pi_{\theta,\phi}^{overall}(\boldsymbol{a}\,|\,\boldsymbol{x})}
\newcommand{\firstpolicy}{\pi_{\theta}^{1st}(\boldsymbol{a}_{1:k}\,|\,\boldsymbol{x})}
\newcommand{\secondpolicy}{\pi_{\phi}^{2nd}(\boldsymbol{a}_{k+1:L}\,|\,\boldsymbol{x},\boldsymbol{a}_{1:k})}

\newcommand{\loggingfirstpolicy}{\pi_{0}^{1st}(\boldsymbol{a}_{1:k}|\boldsymbol{x})}
\newcommand{\loggingfirstpolicyc}{\pi_{0}^{1st}(\boldsymbol{c}_{1:k}|\boldsymbol{x})}
\newcommand{\loggingsecondpolicy}{\pi_{0}^{2nd}(\boldsymbol{a}_{k+1:L}|\boldsymbol{x},\boldsymbol{a}_{1:k})}

\newcommand{\truePG}{V(\pi_{\theta,\phi}^{overall})}
\newcommand{\ips}{\widehat{V}_{\mathrm{IPS}}(\pi_\theta; \mathcal{D})}

\newcommand{\rpod}{\widehat{V}_{\mathrm{RPOD}}(\pi_{\theta,\phi}^{overall}; \mathcal{D})}

\newcommand{\bx}{\boldsymbol{x}}
\newcommand{\ba}{\boldsymbol{a}}
\newcommand{\bb}{\boldsymbol{b}}
\newcommand{\bc}{\boldsymbol{c}}
\newcommand{\br}{\boldsymbol{r}}

\newcommand{\vecr}{\boldsymbol{r}}

\newcommand{\indicator}[1]{\mathbb{I}\{#1\}}

\AtBeginDocument{
  \providecommand\BibTeX{{
    \normalfont B\kern-0.5em{\scshape i\kern-0.25em b}\kern-0.8em\TeX}}}

\copyrightyear{2026}
\acmYear{2026}
\setcopyright{cc}
\setcctype{by}
\acmConference[CIKM '26]{Proceedings of the 35th ACM International Conference on Information and Knowledge Management}{November 07--11, 2026}{Rome, Italy}
\acmBooktitle{Proceedings of the 35th ACM International Conference on Information and Knowledge Management (CIKM '26), November 07--11, 2026, Rome, Italy}
\acmDOI{10.1145/3799682.3840629}
\acmISBN{979-8-4007-2539-5/2026/11}

\ccsdesc[500]{Information systems~Retrieval models and ranking}

\keywords{off-policy learning, ranking policy, inverse propensity score.}

\begin{document}

\title[Efficient Offline Learning of Ranking Policies via Top-$k$ Policy Decomposition]{Efficient Offline Learning of Ranking Policies\\via Top-$k$ Policy Decomposition}

\renewcommand{\shortauthors}{Kishimoto, et al.}

\begin{abstract}
Many recommender systems such as for e-commerce and news platforms aim to provide users with rankings they are likely to interact with. \textit{Off-Policy Learning} (OPL) of ranking policies enables us to learn new ranking policies using only historical logged data. However, ranking settings make OPL remarkably challenging because their action spaces consist of permutations of unique items, being extremely large. Existing methods primarily use either policy- or regression-based approaches. The policy-based approach, which typically uses importance-weighted policy gradients, can suffer from high variance due to large action spaces. The regression-based approach, on the other hand, estimates the expected reward using conventional machine learning methods, avoiding variance issues but potentially suffering from severe bias. To circumvent these issues of existing methods, we propose a new OPL method for ranking, named \textbf{\textit{Ranking Policy Optimization via Top-$k$ Policy Decomposition (R-POD)}}, which combines the policy- and regression-based approaches in an effective fashion. Specifically, R-POD decomposes a ranking policy into a first-stage policy for selecting \textit{top-$k$} actions and a second-stage policy for choosing the bottom actions given the top-$k$ actions. It learns the first-stage policy using a new policy gradient estimator and the second-stage policy via the regression-based approach. This method can substantially reduce variance, since it applies importance weighting only to the top-$k$ actions. We also demonstrate that our policy-gradient estimator for the first-stage policy is unbiased under a \textit{conditional pairwise correctness} condition, which only requires that the expected reward differences of pairs of rankings sharing the same top-$k$ actions can be estimated correctly. Comprehensive experiments illustrate that R-POD provides substantial improvements in OPL for ranking.
\end{abstract}

\maketitle
\section{Introduction}
Intelligent systems in the real-world, such as recommender systems, search engines, and news applications, often present items (e.g., products, news, and jobs) in the form of rankings. In these systems, our goal is often to learn new ranking policies to improve outcomes, using only historical logged data collected by logging policies. 

This learning task is known as \textit{Off-Policy Learning} (OPL). OPL is highly relevant in many practical applications involving automated decision-making regarding ranking interface~\cite{saito2021counterfactual,li2018offline}.

The main approaches to OPL include policy-based and regression-based methods~\cite{saito2021counterfactual}. The policy-based approach learns new policies by estimating the policy gradient, often through importance-weighting~\cite{precup2000eligibility,dudik2011doubly}. Although this approach can be based on unbiased policy gradients and learns effective policies with sufficient logged data, it can be sample-inefficient particularly under large action spaces~\cite{peng2023offline,sachdeva2023off}. In ranking settings, in particular, where the action space corresponds to all possible permutations of items, most policy gradient estimators collapse due to extremely high variance~\cite{kiyohara2022doubly,saito2023off,saito2022off,peng2023offline}. To mitigate the variance issue caused by importance weighting in ranking setups, several methods introduce assumptions about user behavior such as independence~\cite{li2018offline} and cascade~\cite{mcinerney2020counterfactual}. While these methods succeed in reducing variance, restrictive assumptions can lead to large bias in policy gradient estimations~\cite{kiyohara2023off,kiyohara2022doubly,mcinerney2020counterfactual}. On the other hand, the regression-based approach learns the reward function and selects a ranking with the highest predicted reward. It can avoid variance issues but is known to suffer from high bias due to the difficulty of accurately modeling the rewards of every unique ranking in the action space.

To address the bias and variance issues caused by ranking action spaces, we develop a novel OPL algorithm for ranking called \textbf{\textit{Ranking Policy Optimization via Top-$k$ Policy Decomposition (R-POD)}}. The crux of R-POD is to decompose a ranking policy into a first-stage policy that chooses the best top-$k$ actions for each context and a second-stage policy that selects the best bottom actions given the top-$k$ actions. Leveraging this decomposition of a ranking policy, we learn the first-stage policy through the policy-based approach with a novel policy gradient estimator for ranking, called the R-POD gradient estimator. The R-POD gradient estimator leverages importance weighting in the top-$k$ action space to account for the effect of top-$k$ actions and a reward regression model to consider the effect of the bottom actions. We demonstrate that the R-POD gradient estimator is unbiased under a conditional pairwise correctness (CPC) condition, which only requires that the regression model accurately preserves the relative expected reward differences of rankings sharing the same top-$k$ actions. The second-stage policy of R-POD is then learned by the regression-based approach. We show that the second-stage policy can be based on a reward model that is used as a part of the R-POD gradient estimator for the first stage-policy, since the CPC condition ensures that the second-stage policy performs optimally in terms of choosing the bottom actions.

Compared to existing policy-based methods for ranking, our R-POD gradient estimator applies importance weighting only to the top-$k$ actions, substantially reducing variance, as we theoretically demonstrate. In addition, R-POD avoids introducing large bias to achieve substantial variance reduction, as it does not introduce any assumptions about user behavior. Moreover, compared to existing regression-based approaches, R-POD relaxes the modeling requirement regarding the reward function. Specifically, it only needs to accurately learn the relative value differences between pairs of unique rankings that have the same top-$k$ actions, a condition that is generally milder than learning the global expected rewards for every unique ranking. 
Comprehensive experiments on both synthetic and real-world ranking data illustrate that our R-POD algorithm performs more effectively than existing policy- and regression-based methods for a variety of experiment settings.

\section{Off-Policy Learning for Ranking}
This section formulates the problem of OPL for ranking policies and describes existing methods and their limitations. 

\subsection{Problem Formulation}
In our formulation of OPL for ranking, $\bx\in \calX\subseteq{\mathbb{R}^{d_x}}$ denotes a $d_x$-dimensional context vector (e.g., user demographics) drawn i.i.d. from an unknown distribution $p(\bx)$. The finite set of discrete (unique) actions is denoted as $\calA$, with $a\in\calA$ corresponding to an action like a single movie, song, news article, or product. Let $\ba = (a_1, a_2, \ldots, a_l, \ldots, a_L)$ be a ranking action vector, where $L$ denotes the length of the ranking. The function $\pi : \calX \to \Delta(\prod_{L}(\calA))$ is referred to as a \textit{ranking policy}, with $\prod_{L}(\calA)$ indicating the set of $L$-permutations of $\calA$, i.e., the ranking action space. Furthermore, $\br = (r_1, r_2, \ldots, r_l, \ldots, r_L)$ represents a reward vector, sampled from an unknown conditional distribution $p(\br|\bx,\ba)$, where $r_l$ is the reward observed at the $l$-th position. The effectiveness of a policy $\pi$ is measured through its \textit{value}, which is defined as follows.
\begin{align}
    \label{V_pi}
    V(\pi) \coloneq \mE_{p(\bx)\pi(\ba|\bx)}\left[\sum_{l=1}^{L} \alpha_{l}q_l(\bx,\ba)\right],
\end{align}
where $q_l(\bx,\ba) \coloneq \mE\left[r_l| \bx,\ba\right]$ represents the \textit{position-wise} expected reward function. Here, $\alpha_l$ is a non-negative weight given to each position. This definition of policy value in Eq. (\ref{V_pi}) can represent various additive ranking metrics. For instance, when $\alpha_l \coloneq 1/\log_{2}{(l+1)}$, it represents the discounted cumulative gain (DCG).

The logged data we can use for performing OPL is of the form: 
\begin{align*}
   \mathcal{D} \coloneq \{(\bx^{(i)},\ba^{(i)},\br^{(i)})\}_{i=1}^n,  
\end{align*}

which contains \textit{n} independent observations drawn from the logging policy $\pi_0$ as $(\bx,\ba,\br) \sim p(\bx)\pi_0(\ba|\bx)p(\br|\bx,\ba)$. In OPL of ranking policies, using only $\mathcal{D}$, we aim to optimize a ranking policy $\pi_{\theta}$, which is parameterized by $\theta$, to maximize the policy value as
\begin{align}
    \theta^* = \argmax_{\theta \in \Theta}\, V(\pi_{\theta}).
\end{align}
There exist two typical approaches to solve this policy learning task, namely the policy- and regression-based approaches, as described in detail in the following.

\subsection{Limitations of Existing Methods}

Firstly, \textbf{the policy-based approach} aims to learn the policy parameter $\theta$ via gradient ascent, $\theta_{t+1} \leftarrow \theta_{t} + \eta \nabt V(\pi_\theta)$, where
\begin{align}
    \nabt V(\pi_\theta) := \mE_{p(\bx)\pi_{\theta}(\ba|\bx)}
    \left[ \left(\sum_{l=1}^{L}\alpha_{l} q_l(\bx,\ba) \right)
    \nabt \log{\pi_\theta(\ba|\bx )} \right] \label{eq:true-pg}
\end{align}
is called the policy gradient (we can derive it via the log-derivative trick, i.e., $\nabt  \pi_{\theta} = \pi_{\theta} \nabt  \log\pi_{\theta}$). The problem here is that we do not know the true policy gradient $\nabt V(\pi_\theta)$ since we do not know the true reward functions $\{q_l(\bx,\ba)\}_{l=1}^L$. Therefore, we need to estimate the policy gradient with only available logged data $\mathcal{D}$. A standard approach to do it is to apply \textit{inverse propensity scoring (IPS)} as
\begin{align}
    \label{eq:ips}
    \nabt\ips 
    \coloneq \frac{1}{n}\sum_{i=1}^{n} w(\bx^{(i)},\ba^{(i)})
\left(\sum_{l=1}^{L}\alpha_{l}r_{l}^{(i)}\right) \score(\bx^{(i)},\ba^{(i)}),
\end{align}
where $w(\bx,\ba) \coloneq \pi_\theta(\ba|\bx) / \pi_0(\ba|\bx)$ is called the \textit{ranking-level} importance weight, which is defined as the ratio of probabilities that a unique ranking $\ba$ is chosen under two different policies. We also use $\score(\bx,\ba) \coloneq \nabt \log{\pi_\theta(\ba|\bx)}$ to denote the policy score function in Eq.~\eqref{eq:ips}. It is widely known that the IPS gradient estimator given above is unbiased under the full support condition.
\begin{condition}
    (Full Support) The logging policy $\pi_0$ is said to have full support if $\pi_0(\ba|\bx) > 0$ for all $\ba \in \prod_{L}(\calA)$ and $\bx \in \calX$. \label{ass:full_support}
\end{condition}
Unfortunately, in the ranking problem, the full support condition is often hard to guarantee due to a large number of unique rankings~\cite{saito2023off}, potentially resulting in substantial bias for IPS~\cite{sachdeva2020off,felicioni2022off,saito2022off}. We also describe the variance of the IPS gradient estimator to highlight its more critical issue in the following.
\begin{align}
&n\mV_{\mathcal{D}}\left[\nabt\widehat{V}_{\mathrm{IPS}}^{(j)}(\pi_\theta ; \mathcal{D}) \right] \notag \\
    &=\mE_{p(\bx)\pi_0(\ba|\bx)}\left[\left(w(\bx,\ba)s_{\theta}^{(j)}(\bx,\ba)\right)^2 \sigma^2(\bx,\ba)\right]\notag\\
    &\quad+\mE_{p(\boldsymbol{{x}})}\left[\mV_{\pi_0(\ba|\bx)}\left[w(\bx,\ba)\left(\sum_{l=1}^{L}\alpha_{l}q_{l}(\bx,\ba)\right)s_{\theta}^{(j)}(\bx,\ba)\right]\right] \notag\\    
    &\quad+\mV_{p(\bx)}\left[\mE_{\pi_0(\ba|\bx)}\left[w(\bx,\ba)\left(\sum_{l=1}^{L}\alpha_{l}q_{l}(\bx,\ba)\right)s_{\theta}^{(j)}(\bx,\ba)\right]\right], \label{eq:ips-variance}
\end{align}
where $\sigma^2(\bx,\ba) \coloneq \mV[(\sum_{l=1}^{L}\alpha_{l}r_{l})|\bx,\ba]$ is the conditional variance of the ranking metric and $s_{\theta}^{(j)}(\bx,\ba)$ is the \textit{j}-th dimension of the score function. Eq.~\eqref{eq:ips-variance} indicates that the variance can become excessively large when the importance weights $w(\bx,\ba)$ take a large value. Indeed, $\pi_0(\ba|\bx)$ often becomes extremely small, particularly when it satisfies the requirement of the full support condition in ranking action spaces, leading to large variation of the weights. 

To deal with the variance issue of IPS, we can possibly apply the \textit{Doubly Robust (DR)} estimator~\cite{Dud_k_2014}, which uses a reward estimator $\hat{f}(\bx, \ba)\approx \sum_{l=1}^{L}\alpha_{l} q_l(\bx,\ba) $ as a control variate. 
DR can reduce the estimation variance compared to the IPS gradient estimator. 
However, DR still suffers from extremely high variance due to the use of ranking-level importance weighting~\cite{kiyohara2022doubly}.

To address the high variance caused by ranking-level importance weighting, some methods introduce assumptions about user behavior~\cite{li2018offline,mcinerney2020counterfactual,kiyohara2022doubly,kiyohara2023off}. For instance, \citet{li2018offline} and \citet{liu2022practical} leverage the independence assumption, which assumes that users interact with the item presented at each position independently from the other items, and propose the Independent IPS (IIPS) estimator. In contrast, Reward-interaction IPS (RIPS) estimator is based on the cascade assumption, which posits that users interact with items one by one from the top position~\cite{mcinerney2020counterfactual}.
These estimators are unbiased under their underlying assumptions and substantially reduce variance compared to IPS and DR. However, these gradient estimators can exhibit high bias under the violation of their respective behavior assumptions~\cite{kiyohara2023off}.

Secondly, \textbf{the regression-based approach} estimates the reward function as 
$\hat{q}_l(\bx,\ba) \approx q_l(\bx,\ba), \, \forall l\,$ using conventional supervised machine learning methods. It then converts the estimated reward functions $\{\hat{q}_l(\bx,\ba) \}_{l=1}^L$ into a ranking policy, for example, by applying the argmax operator as below.
\begin{align*}
    \pi (\ba \,|\,\bx) :=  \left\{
    \begin{array}{ll}
        1 & (\ba = \argmax_{\ba'\in \Pi_L(\calA) } \sum_{l=1}^L \alpha_l \hat{q}_l(\bx,\ba')) \\
        0 & (\text{otherwise})
    \end{array}
    \right.
\end{align*}
Regarding the variance, this approach is superior to the policy-based approach as it does not involve importance weighting. However, it can suffer from high bias due to the difficulty of accurately regressing the expected rewards for every unique ranking, i.e., $\forall \ba \in \Pi_L(\calA)$, based only on partial feedback in the logged data $\calD$.

It is worth noting that there exists a relevant but orthogonal research direction called \textit{Unbiased Learning-to-Rank} (ULtR)~\cite{joachims2017unbiased,ai2018unbiased,wang2018position}, which aims to optimize the ranking of items based on implicit feedback. The typical challenge of ULtR is how to deal with position bias. On the other hand, OPL considers not just position bias but also selection bias coming from the logging policy $\pi_0$. This point differentiates ranking OPL technically and substantially from ULtR, and our focus is solely on ranking OPL. 

As discussed in this section, classic approaches to OPL are ineffective in the ranking setup. To achieve more efficient OPL even in the ranking problem, the following develops a novel algorithm that circumvents the variance issue of the policy-based approach and the bias issue of the regression-based approach simultaneously.

\begin{figure}[t]
     \centering
     \includegraphics[scale=0.27]{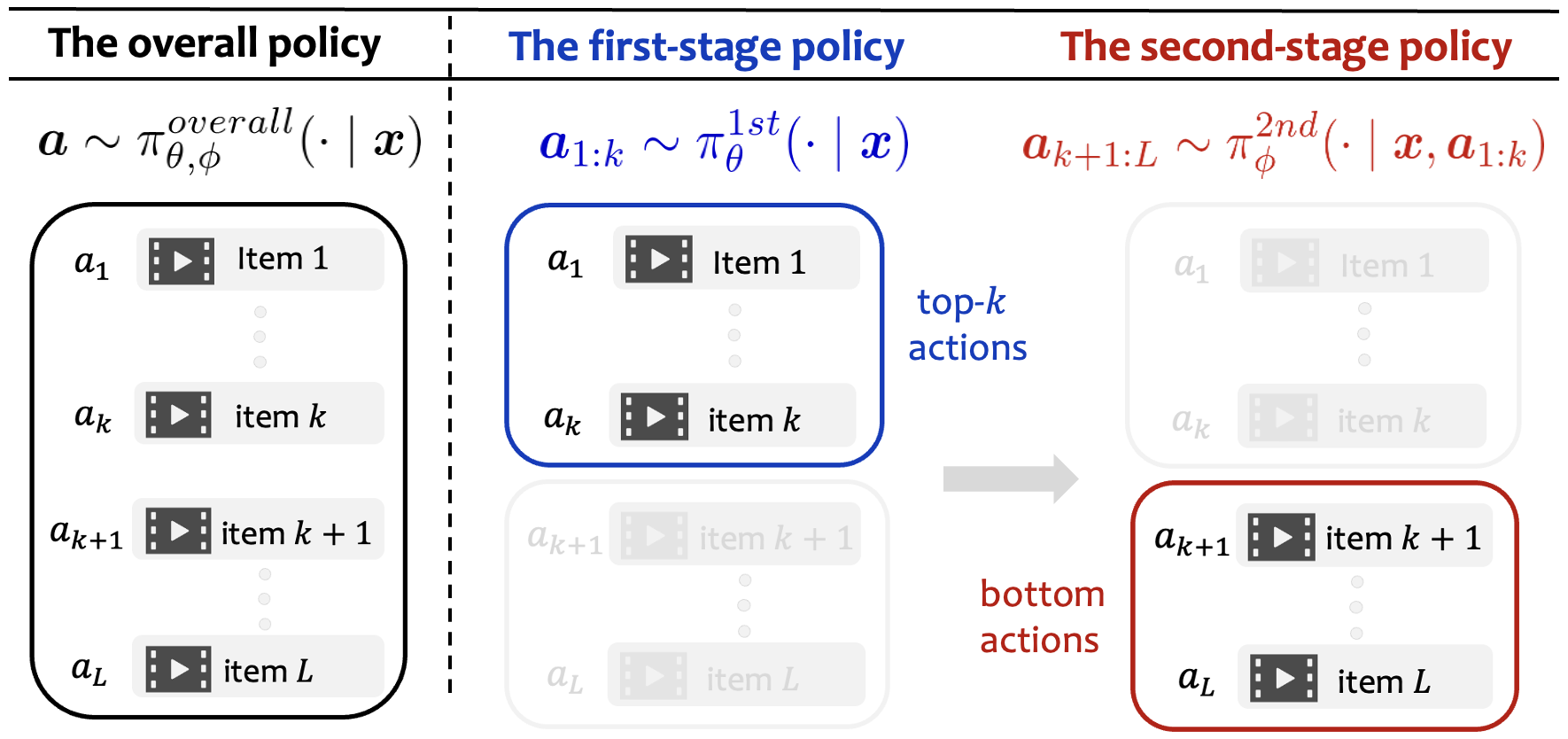}
     \caption{R-POD leverages the concept of top-$k$ policy decomposition in Eq.~\eqref{eq:policy-decomposition} to learn better ranking policies offline.}
     \label{topk decomposition}
     \vspace{-3mm}
\end{figure}

\section{The ``R-POD'' Algorithm}
This section introduces a novel OPL method called \textbf{R-POD}. Its core concept involves decomposing a ranking policy into two components: a first-stage policy and a second-stage policy, as follows.

\begin{tcolorbox}
\large{\textbf{The Top-$k$ Policy Decomposition:}}\normalsize{}
\begin{align}
    \overallpolicy = \firstpolicy \cdot \secondpolicy, 
    \label{eq:policy-decomposition}
\end{align}
where $\ba_{k_1:k_2} \coloneq (a_{k_1},a_{k_1+1},\cdots,a_{k_2-1},a_{k_2})$. As depicted in Figure~\ref{topk decomposition}, the first-stage policy is a segment of the overall policy responsible for selecting the top-$k$ actions ($\ba_{1:k}$). In contrast, the second-stage policy is a segment to choose the bottom actions ($\ba_{k+1:L}$), conditional on the top-$k$ actions already sampled by the first-stage policy.
\end{tcolorbox}

Higher-ranked items are more crucial for producing effective rankings. Hence, we consider optimizing the first-stage policy using a method with low bias, and then the second-stage policy using a method with low variance to control the overall variance of the algorithm. With this idea in mind, our R-POD algorithm for ranking learns an overall policy in two separate stages. In the first stage, it optimizes a first-stage policy through the policy-based approach within the top-$k$ action space. By applying importance weighting exclusively to the top-$k$ action space, which is considerably smaller than the entire ranking space, we can learn the first-stage policy with significantly reduced variance. Then, we optimize the second-stage policy through the regression-based approach. By applying the regression model only to the remaining action space, conditional on a set of top-$k$ actions, we can learn the second-stage policy with smaller bias than the conventional regression-based approach while taking advantage of its low variance. In the following, we describe how to learn first- and second-stage policies from the logged data $\calD$ in order to improve the value of the overall policy.

\subsection{Optimizing the First-stage Policy $\pi_{\theta}^{1st}$}
We first consider optimizing the first-stage policy $\pi_{\theta}^{1st}$, parameterized by $\theta$, via the policy-based approach given a second-stage policy $\pi_{\phi}^{2nd}$. We consider learning the first-stage policy under a pre-trained second-stage policy because the overall policy is dependent both on the first- and second-stage policies. This implies that the optimal first-stage policy becomes different given a different second-stage policy, as we will describe below. 

Given that our ultimate goal is training a better overall policy, we should train the first-stage policy so that the overall policy $\pi_{\theta,\phi}^{overall}$ is improved.
Hence, given a second-stage policy $\pi_{\phi}^{2nd}$, we aim to update the first-stage policy parameter $\theta$ as
\begin{align}
    \theta_{t+1} \leftarrow \theta_{t} + \eta \nabt \truePG
\end{align}
The following derives the policy gradient for the overall policy regarding the first-stage policy parameter, i.e., $\nabt \truePG$.

\begin{proposition} (The overall policy gradient)
    The policy gradient for the overall policy regarding the first-stage policy parameter, i.e., $\nabt \truePG$, is given as follows.
    \begin{align}
        \label{eq:truePG}
        \nabt \truePG &=\mE_{p(\bx)\pi_{\theta}^{1st}(\ba_{1:k}|\bx)}\left[q^{\pi_{\phi}^{2nd}}(\bx,\ba_{1:k}) \score(\bx,\ba_{1:k}) \right],
    \end{align}
    where $q^{\pi_{\phi}^{2nd}}(\bx,\ba_{1:k}) \coloneq \mE_{\pi_{\phi}^{2nd}
    }\left[\sum_{l=1}^{L}\alpha_{l}q_{l}(\bx,\ba)\right]$ denotes the value of top-$k$ actions under $\pi_{\phi}^{2nd}$ and 
    $\score(\bx,\ba_{1:k}) \coloneq \nabt \log{\pi_{\theta}^{1st}(\ba_{1:k}|\bx)}$ is the policy score function of the first-stage policy. See Appendix~\ref{derivation of overall gra} for the proof.
    \label{prop:overall policy gradient}
\end{proposition}

Proposition~\ref{prop:overall policy gradient} suggests that if the first-stage policy can choose the top-$k$ actions that are evaluated highly according to the function $q^{\pi_{\phi}^{2nd}}(\bx,\ba_{1:k})$, we can improve the effectiveness of the overall policy. An interesting observation here is that the top-$k$ actions that the first-stage policy should choose are different given different second-stage policies as implied by the fact that the function $q^{\pi_{\phi}^{2nd}}(\bx,\ba_{1:k})$ is dependent on $\pi_{\phi}^{2nd}$. This is indeed the reason why we consider training the first-stage policy given a (pre-trained) second-stage policy. 

As discussed, the true $\nabt \truePG$ would lead to an improved overall policy, however, we cannot know the ground-truth policy gradient given in Eq.~\eqref{eq:truePG} due to the inability to know $q^{\pi_{\phi}^{2nd}}(\bx,\ba_{1:k})$, so we have to estimate it using logged data to train the first-stage policy. To achieve this, we propose a new policy gradient estimator, called the \textbf{R-POD gradient estimator}, defined as follows.
\begin{align}
    &\nabt\rpod \notag \\
    &\coloneq 
    \frac{1}{n} \sum_{i=1}^{n} \bigg\{ w(\bx^{(i)},\ba_{1:k}^{(i)})
    \left(\sum_{l=1}^{L}\alpha_{l}r_{l}^{(i)} - \hat{f}(\bx^{(i)},\ba^{(i)})\right)\notag \\
    &\hspace{60mm}\times\score(\bx^{(i)},\ba_{1:k}^{(i)})\notag\\
    &\qquad\quad + \mE_{\pi_{\theta}^{1st}(\ba_{1:k}|\bx^{(i)})}
    \left[\hat{f}^{\pi_{\phi}^{2nd}}(\bx^{(i)},\ba_{1:k})\score(\bx^{(i)},\ba_{1:k})\right]\bigg\}, \label{eq:rpod-pg}
\end{align}
where $w(\bx,\ba_{1:k}) \coloneq \frac{\pi_{\theta}^{1st}(\ba_{1:k}|\bx)}{\pi_{0}^{1st}(\ba_{1:k}|\bx)}$ is the \textit{top-$k$ importance weight}. Specifically, the first term of Eq.~\eqref{eq:rpod-pg} estimates the value of top-$k$ actions $\ba_{1:k}$ via importance weighting and the second term deals with the value of bottom actions ($\ba_{k+1:L}$) via the regression model $\hat{f}(\bx,\ba)$. Since the top-$k$ importance weight considers only the difference in probabilities of choosing the top-$k$ actions between policies, it is expected to have much lower variance than existing gradient estimators such as IPS, DR, and RIPS. Additionally, it does not introduce any assumption on user behavior such as independence~\cite{li2018offline} or cascade~\cite{mcinerney2020counterfactual}, so it does not produce large bias regarding the violation of such assumptions. Note that we will discuss how we should optimize the regression model $\hat{f}(\bx,\ba)$ based on the analysis of the R-POD estimator provided below.

As a theoretical analysis, we first characterize the bias of the R-POD gradient estimator under the full ``top-$k$'' support condition, which is less restrictive than the full support condition (Condition~\ref{ass:full_support}) needed for the unbiasedness of IPS.
\begin{condition}
    \label{eq:full_top-k_support}
    (Full top-$k$ support) The logging policy $\pi_0$ satisfies full top-$k$ support if $\pi_0(\ba_{1:k}|\bx) > 0$ for all $\ba \in \prod_{L}(\calA)$ and $\bx\in \calX$.
\end{condition}

\begin{theorem}
    (Bias of the R-POD gradient estimator) When Condition \ref{eq:full_top-k_support} is true, the R-POD gradient estimator has the following bias for a given regression model $\hat{f}(\bx,\ba)$.
    \begin{align}
         \!\!\!&Bias(\nabt\rpod) \notag \\
         \!\!\!&= \mE_{\pi_{0}^{1st}(\bx,\bc_{1:k})}
        \bigg[\smashoperator[r]{\sum_{\ba<\bb : \ba_{1:k}=\bb_{1:k}=\bc_{1:k}}}
        \pi_{0}^{2nd}(\ba_{k+1:L}|\bx,\bc_{1:k})
        \pi_{0}^{2nd}(\bb_{k+1:L}|\bx,\bc_{1:k})\notag\\
        \!\!\!& \times{(\Delta_{q}(\bx,\ba,\bb)-\Delta_{\hat{f}}(\bx,\ba,\bb))}
        {\left(\frac{\pi_{\theta}(\bb|\bx,\bc_{1:k})}
        {\pi_{0}(\bb|\bx,\bc_{1:k})}
        \!-\!\frac{\pi_{\theta}(\ba|\bx,\bc_{1:k})}
        {\pi_{0}(\ba|\bx,\bc_{1:k})}\right)} \notag \\
        &\hspace{50mm}\times\score(\bx,\bc_{1:k}) \bigg], 
        \label{eq:bias_of_rpod-pg} 
    \end{align}
    where $\ba,\bb \in \prod_{L}(\calA)$. $\Delta_{q}(\bx,\ba,\bb) \coloneq q(\bx,\ba)-q(\bx,\bb)$ represents the difference of the expected rewards between a pair of rankings $\ba$ and $\bb$ given $\bx$, which we call \textbf{the relative value difference of rankings}. $\Delta_{\hat{f}}(\bx,\ba,\bb) \coloneq \hat{f}(\bx,\ba)-\hat{f}(\bx,\bb)$ is a relative value of rankings between $\ba$ and $\bb$ given $\bx$ estimated by the regression model $\hat{f}(\bx,\ba)$.
    \label{thm:bias_of_rpod-pg}
\end{theorem}

The most important factor in Eq.~\eqref{eq:bias_of_rpod-pg} is $\Delta_{q}(\bx,\ba,\bb) - \Delta_{\hat{f}}(\bx,\ba,\bb)$, which implies that, when a regression model $\hat{f}(\bx,\ba)$ accurately preserves the relative value differences of rankings containing the same top-$k$ actions, the bias of the R-POD gradient estimator becomes small. 
Intuitively, the R-POD gradient estimator already unbiasedly estimates the value of top-$k$ actions via its top-$k$ importance weighting, and thus it is sufficient for the regression model to identify only relative value differences of rankings given the same top-$k$ actions to make the gradient estimator unbiased. Moreover, Theorem~\ref{thm:bias_of_rpod-pg} implies that the R-POD gradient estimator becomes unbiased under the following Conditional Pairwise Correctness (CPC) condition.
\begin{condition}
    \label{ass:pair_correct}
    (Conditional Pairwise Correctness; CPC) A regression model $\hat{f}(\bx,\ba)$ satisfies conditional pairwise correctness if $\Delta_{q}(\bx,\ba,\bb) = \Delta_{\hat{f}}(\bx,\ba,\bb)$ for all $x \in \calX$ and $\ba,\bb \in \prod_L(\calA)$ s.t. $\ba_{1:k} = \bb_{1:k}$.
\end{condition}

\begin{corollary}
    Under Conditions~\ref{eq:full_top-k_support} and~\ref{ass:pair_correct}, the R-POD gradient estimator is unbiased, i.e., 
    $\mE_{\mathcal{D}}[\nabt \rpod] = \nabt \truePG$.
\end{corollary}

Thus, the R-POD gradient estimator can be unbiased when the regression model satisfies conditional pairwise correctness, which is less restrictive than aiming for correctly estimating the global reward functions $\{q_l(\bx,\ba)\}_{l=1}^L$ like implicitly assumed for the regression-based approach. The above bias analysis also implies that we should ideally optimize the regression model $\hat{f}(\bx,\ba)$ so that it preserves the relative value differences to minimize the bias of the resulting gradient estimator.

Next, the following calculates the variance of R-POD to show its relation with the accuracy of the regression model.

\begin{proposition} \label{prop:rpod-variance}
    (Variance of the R-POD gradient estimator) Under Conditions~\ref{eq:full_top-k_support} and~\ref{ass:pair_correct}, the variance of the R-POD gradient estimator is given by
    \begin{align}
    \label{eq:rpod-variance}
    &n\mathbb{V}_{\mathcal{D}}(\nabla_\theta\widehat{V}_{\mathrm{RPOD}}^{(j)}(\pi_{\theta,\phi}^{overall}; \mathcal{D}))\notag\\
    &= \mE_{p(\bx)\pi_0(\ba|\bx)}\left[\left(
    w(\bx,\ba_{1:k})
    s_{\theta}^{(j)}(\bx,\ba_{1:k})\right)^2\sigma^2(\bx,\ba)\right]\notag\\
    &\quad+ \mE_{p(\bx)}\left[\mV_{\pi_0(\ba|\bx)}\left[
    w(\bx,\ba_{1:k})
    \Delta_{q,\hat{f}}(\bx,\ba)
    s_{\theta}^{(j)}(\bx,\ba_{1:k})\right]\right]\notag\\
    &\quad+ \mV_{p(\bx)}\left[\mE_{\pi_{\theta}^{1st}(\ba_{1:k}|\bx)}
    \left[q^{\pi_{\phi}^{2nd}}(\bx,\ba_{1:k})
    s_{\theta}^{(j)}(\bx,\ba_{1:k})\right]\right],
\end{align}
where $\Delta_{q,\hat{f}}(\bx,\ba) \coloneq q(\bx,\ba)-\hat{f}(\bx,\ba)$ is the \textit{estimation error} of $\hat{f}(\bx,\ba)$ against the global expected reward function $q(\bx,\ba)$.
\end{proposition}

Proposition~\ref{prop:rpod-variance} suggests that, in terms of variance minimization, we should optimize the regression model in a way that minimizes $|\Delta_{q,\hat{f}}(\bx,\ba)|$ compared to minimizing $|\Delta_{q} (\bx,\ba,\bb) - \Delta_{\hat{f}} (\bx,\ba,\bb)|$ for the bias. Based on the theoretical observations, an ideal strategy to optimize the regression model $\hat{f}(\bx,\ba)$ would be a two-step procedure to directly optimize the bias and variance of the R-POD gradient estimator in each step. 
Specifically, the first step focuses on minimizing the bias by optimizing a pairwise regression function $\hat{h}_{\phi}(\bx,\ba)$ towards accurately estimating the relative reward differences $|\Delta_{q} (\bx,\ba,\bb) - \Delta_{\hat{f}} (\bx,\ba,\bb)|$, which needs pairwise logged data. The second step then aims for variance minimization via minimizing $|\Delta_{q,\hat{f}}(\bx,\ba)|$. More details of the ideal procedure of two-step regression can be found in Appendix~\ref{two-step procedure}.

Note that we do not expect this ideal two-step procedure to be always feasible in practice due to its need of pairwise logged data and implementation costs. Even if it is impractical, we can still employ a conventional regression to estimate the expected absolute reward to construct the regression model. This can be done by optimizing a parameterized function $\hat{f}_{\phi}: \calX \times \Pi_L(\calA) \rightarrow \mathbb{R}$ via:
\begin{align}
    \min_{\phi} \sum_{(\bx,\ba,\br) \in \calD} \ell_f \big(\br, \hat{f}_{\phi} (\bx,\ba) \big), 
    \label{eq:one_step_regression}
\end{align}
 and $\hat{f}_{\phi}(\bx,\ba)$ is used in Eq.~\eqref{eq:rpod-pg}. $\ell_f$ is a loss function to measure the accuracy of $\hat{f}_{\phi}(\bx,\ba)$, which can be defined, for example, as $\ell_f \big(\br, \hat{f}_{\phi} (\bx,\ba) \big) = \big(\sum_{l=1}^L \alpha_l r_l - \hat{f}_{\phi} (\bx,\ba) \big)^2$. 
 Even with this practical and simple procedure, the R-POD gradient estimator retains advantages over existing policy gradient estimators by its significant variance reduction. Section~\ref{sec:experiment} empirically demonstrates that R-POD performs more effectively than existing approaches with this practical regression procedure.

\subsection{Optimizing the Second-stage Policy $\pi_{\phi}^{2nd}$}
The objective of the second-stage policy $\pi_{\phi}^{2nd}$ is to identify the bottom actions to optimize the expected ranking metric given the top-$k$ actions chosen by the first-stage policy $\pi_{\theta}^{1st}$. Specifically, the second-stage policy should rank the bottom actions that yield the highest value among rankings containing the same top-$k$ actions. From the regression procedure mentioned earlier, we have already obtained $\hat{f}_{\phi}(\bx,\ba)$ to estimate the reward function via the regression-based approach. This allows us to readily define the second-stage policy based on $\hat{f}_{\phi}(\bx,\ba)$, for example, as follows.
\begin{align}
    \label{train pi_2}
    \pi_{\phi}^{2nd}(\ba_{k+1:L}|\bx,\ba_{1:k}) \coloneq
   \begin{cases}
      1\quad (\ba = \argmax_{\ba':\ba'_{1:k}=\ba_{1:k}} \hat{f}_{\phi}(\bx,\ba'))\\
      0\quad (otherwise)
   \end{cases}
\end{align}
When the regression model satisfies the CPC condition, the above second-stage policy is optimal because CPC ensures a regression model to accurately estimate the relative value of rankings that share the same top-$k$ actions. This is a more relaxed modeling requirement compared to the existing regression-based approach.

\subsection{The Overall R-POD Algorithm}
The overall process of the R-POD algorithm is conducted as follows. 
First, we construct the regression model $\hat{f}_{\phi}(\bx,\ba)$, for example, via performing Eq.~\eqref{eq:one_step_regression}. We then formulate the second-stage policy $\pi_{\phi}^{2nd}$ based on the regressor $\hat{f}_{\phi}(\bx,\ba)$ as in Eq.~\eqref{train pi_2}. We also use $\hat{f}_{\phi}(\bx,\ba)$ and optimize the first-stage policy $\pi_{\theta}^{1st}$ via iterative gradient ascent using the R-POD gradient estimator in Eq.~\eqref{eq:rpod-pg}. 

Once we obtain first- and second-stage policies via the R-POD algorithm, for an incoming context $\bx$ in the inference phase, we first sample top-$k$ actions from the 1st-stage policy as $\ba_{1:k} \sim \firstpolicy$. We then apply the 2nd-stage policy to rank the bottom actions given the top-$k$ actions as $\ba_{k+1:L} \sim \secondpolicy$. This procedure is equivalent to sampling a ranking from the joint distribution induced by $\firstp$ and $\secondp$, i.e., $\ba \sim \overallpolicy$.

\begin{figure}[t]
     \centering
     \includegraphics[width=0.90\linewidth]{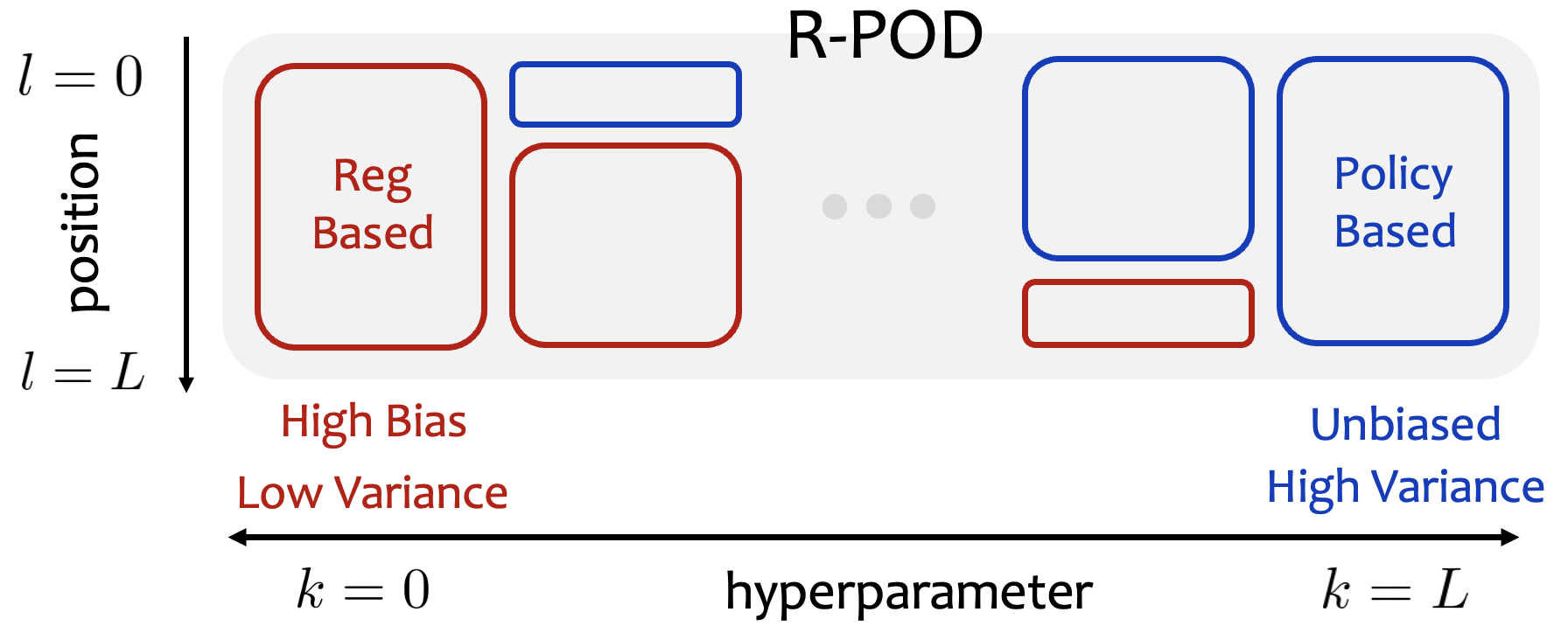}
     \caption{The R-POD algorithm mixes policy- and regression-based approaches via its hyperparameter $k$. When $k=L$, R-POD reduces to the policy-based approach, while it reduces to the regression-based approach with $k=0$.}
     \label{fig:role_of_k}
     \vspace{-3mm}
\end{figure}

\subsection{The Role of Hyperparameter $k$ in R-POD}
The hyperparameter $k$ in R-POD plays a crucial role in deciding the effectiveness of the algorithm. When $k$ is large, the bias of the gradient estimator for the first-stage policy is expected to be small. This is because decreasing the number of bottom actions makes CPC milder.
In the extreme case where $k=L$, the R-POD gradient estimator becomes unbiased irrespective of the accuracy of the regression model because CPC requires nothing. In contrast, the variance of the R-POD gradient estimator may increase because a larger number of top-$k$ actions leads to higher variance in the top-$k$ importance weight. Conversely, when $k$ is small, the variance of the R-POD gradient estimator decreases while its bias increases.
It should be noted that when $k=L$, the first-stage policy becomes identical to the overall policy and thus the R-POD algorithm reduces to the policy-based approach. In contrast, when $k=0$, the second-stage policy becomes identical to the overall policy and thus R-POD reduces to the regression-based approach. This provides an intriguing interpretation of the hyperparameter $k$ as the mixture ratio of the policy- and regression-based approaches in R-POD as described in Figure~\ref{fig:role_of_k}. In practice (and in the following experiments), we can tune this key hyperparameter based on a hold-out estimate of the policy value based on OPE estimators such as IPS or DR.

\section{Empirical Evaluation} \label{sec:experiment}
This section empirically evaluates R-POD on both synthetic and public ranking datasets. Note that, in our experiments, we focused on settings with a unique action space of less than 100 ($|\calA| \le 100$). Existing methods cannot even handle this problem due to their severe variance issues. Indeed, many relevant work around off-policy ranking focus on the problem with a similar or smaller size~\cite{li2018offline,mcinerney2020counterfactual,kiyohara2022doubly,kiyohara2023off,kiyohara2024off}. It is also true that, when $|\calA|=100$ and $L=5$, the number of unique ranking is $|\Pi_L(\calA)| > 10^9$, which is massive. In particular, our real-world experiments reflect this vast ranking space. We believe that OPL for ranking policies with even larger ranking spaces, potentially leveraging structure in $\calA$ as studied by~\cite{saito2022off,saito2023off,sachdeva2023off,cief2024learning,kiyohara2024off}, would be a valuable future direction.

\begin{figure*}[t]
  \centering
  \includegraphics[width=0.85\textwidth]{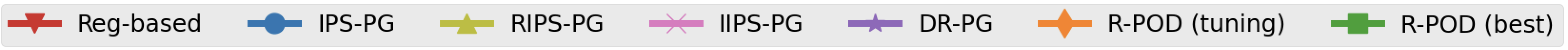}
  \label{fig:legend}
  \vspace{1mm}
  \begin{subfigure}[a]{0.24\textwidth} 
    \centering
    \includegraphics[width=\linewidth]{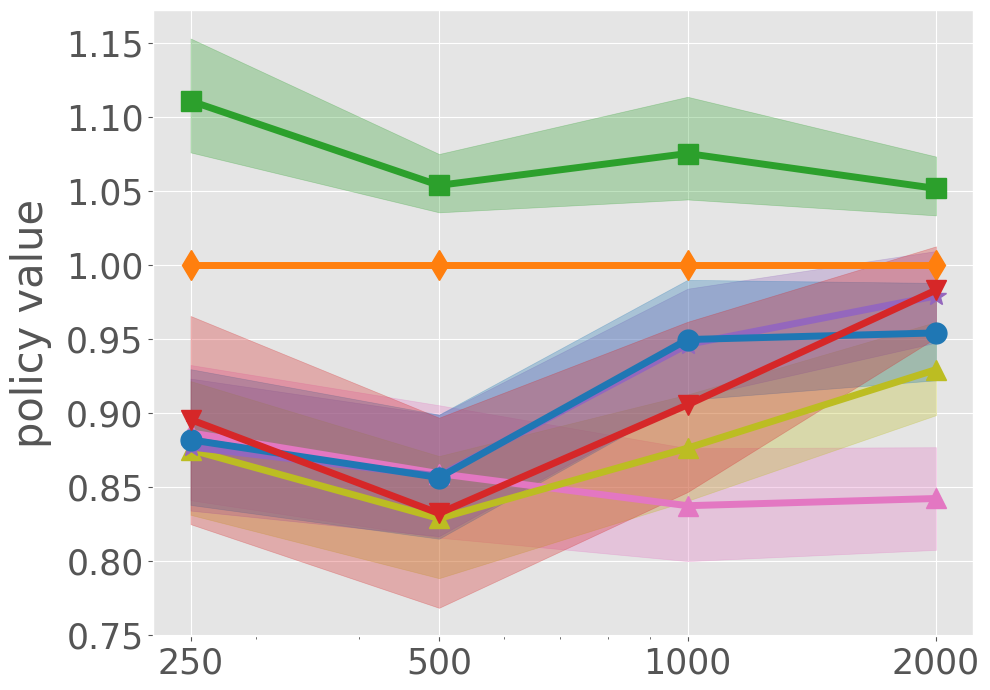}
    \caption{Training data size ($n$)}
    \vspace{1em}
    \label{fig:training_data}
  \end{subfigure}
  \begin{subfigure}[a]{0.24\textwidth}
    \centering
    \includegraphics[width=\linewidth]{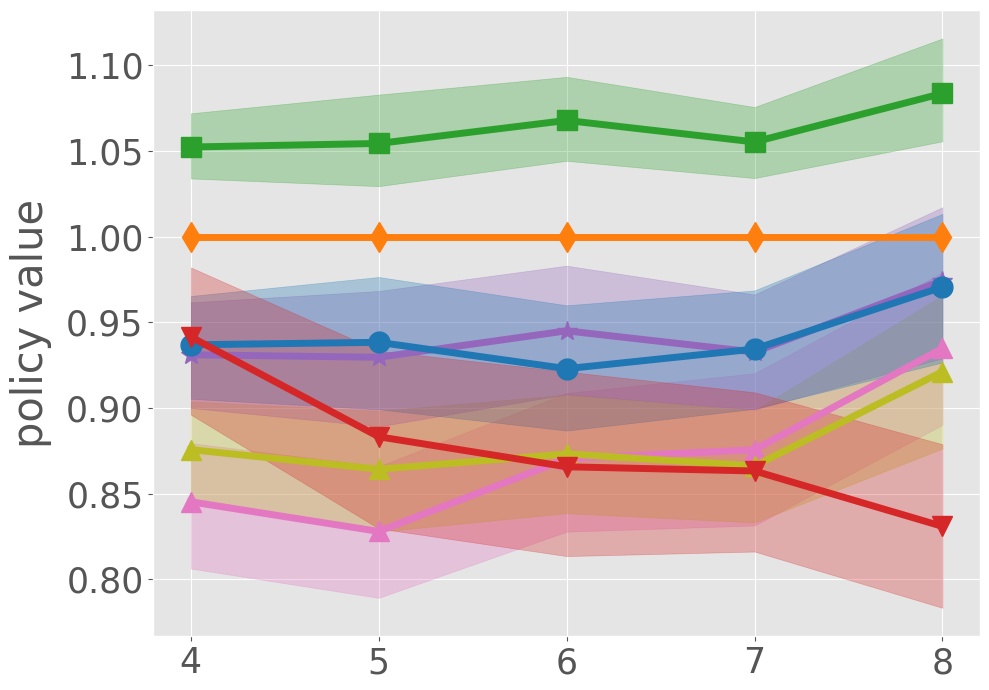}
    \caption{Number of unique actions ($|\calA|$)}
    \vspace{1em}
    \label{fig:unique_action}
  \end{subfigure}
  \begin{subfigure}[a]{0.24\textwidth}
    \centering
    \includegraphics[width=\linewidth]{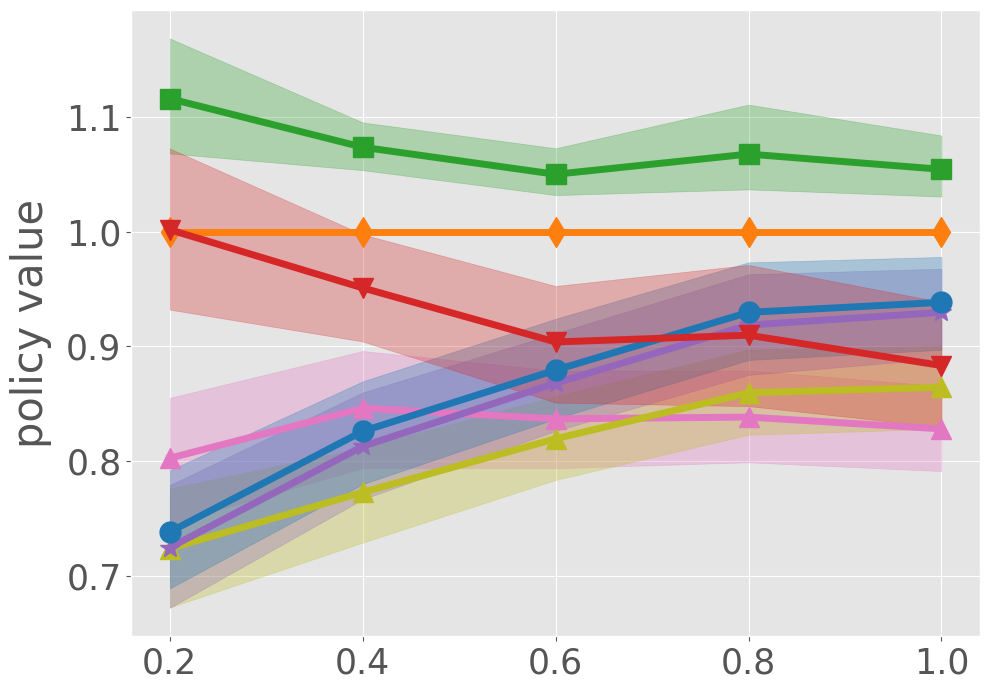}
    \caption{Temperature parameter of the logging policy ($\tau$)}
    \label{fig:behavior_tau}
  \end{subfigure}
  \begin{subfigure}[a]{0.24\textwidth}
    \centering
    \includegraphics[width=\linewidth]{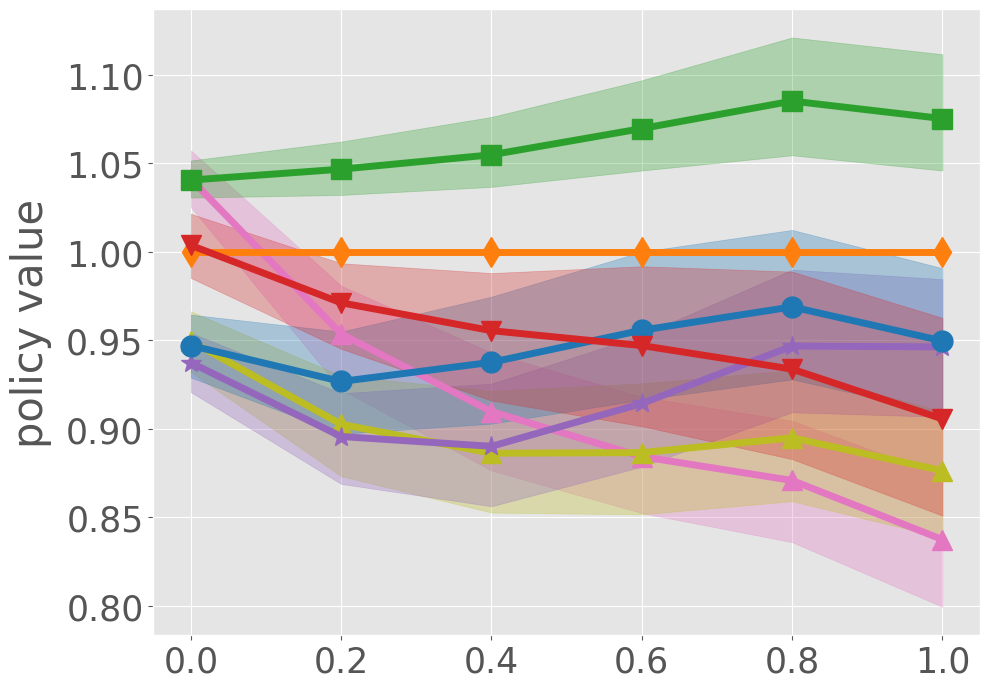}
    \caption{Interaction parameters ($\lambda$)}
    \vspace{1em}
    \label{fig:probability_affected}
  \end{subfigure}
  \vspace{-3mm}
  \caption{Comparing the test policy values (normalized by R-POD (tuning)) of the OPL methods, with varying (a) training data sizes, (b) numbers of actions, (c) temperature parameters of the logging policy, and (d) extents of interactions.} \label{fig:synthetic_results}
\end{figure*}

\subsection{Synthetic Data}
\label{subsec:synthetic}
To generate synthetic datasets, we sample 5-dimensional contexts from the standard normal distribution. 
Then, for each context-ranking pair, we first synthesize the expected reward function for each position $l \,(1\le l \le L)$ in a ranking as 
\begin{align}
    q_l(\bx,\ba) := \tilde{q}_l(\bx, a_l) + F (\bx, \ba), \label{eq:q_func}
\end{align}
where $\tilde{q}_l(\bx, a_l)$ is called the base reward function and defines the value of action $a_l$ presented at the corresponding position $l$, while $F (\bx, \ba)$ depends on the whole ranking $\ba$, introducing interactions and violates typical behavior assumptions like cascade and independence.
Specifically, the former term is defined as $\tilde{q}_l(\bx, a_l):=\theta_{a_l}^{\top} \bx + b_{a_l}$, where $\theta_{a_l}$ is a parameter vector sampled from the standard normal distribution and $b_{a_l}$ is a bias term defined uniquely for action $a_l$.
In contrast, the interaction term is defined as $F (\bx, \ba) = \sum_{m \neq l} \mathbb{W}(a_m, a_l) \indicator{X_l \leq \lambda}$, where $\mathbb{W}(a_m, a_l)$ indicates the effect of action $a_m$ on the reward of action $a_l$. $X_l$ is a random variable sampled from the standard uniform distribution, and $\lambda \in [0, 1]$ is a parameter to control the extent of interaction. We then sample the reward $r_l$ from a normal distribution, whose mean is $q_l(\bx,\ba)$ and standard deviation $\sigma$ is 0.5. 

We define the logging policy that produces the logged data based on the Plackett-Luce model~\cite{plackett1975analysis} as follows.
\begin{align}
    \pi_0(\ba | \bx) = \prod_{l=1}^L 
 \frac{\exp(f_0(\bx, a_l) / \tau)\mathbb{I}[a_l \notin \ba_{1:l-1}]}{\sum_{a' \in \mathcal{A}\backslash \ba_{1:l-1}} \exp(f_0(\bx, a')/ \tau)} \label{eq:plackett-luce}
\end{align}

where $f_0(\bx, a_l) = {\tilde{\theta}_{a_l}}^{\top} \bx + \tilde{b}_{a_l}$ and $\tau$ is a temperature parameter. We sample both ${\tilde{\theta}_{a_l}}$ and $\tilde{b}_{a_l}$ from the standard uniform distribution.

\paragraph{\textbf{Compared Methods.}}
We compare R-POD with IPS-PG, DR-PG, RIPS-PG~\cite{mcinerney2020counterfactual}, IIPS-PG~\cite{li2018offline}, and Regression-based approach (Reg-based). 
To determine the hyperparameter $k$ for R-POD, we perform grid-search in range $k \in [0, L]$ based on the policy value estimated by IPS in a hold-out set. R-POD with data-driven tuning of $k$ is denoted as \textbf{``R-POD (tuning)''} in our experiment results. We also report the results of \textbf{``R-POD (best)''}, which uses the hyperparameter $k$ with the best ground-truth policy value and provides the best achievable value as a reference.

\paragraph{\textbf{Results}}
Figure \ref{fig:synthetic_results} compares the value of policies learned by each OPL method over 100 simulations with different random seeds. Each figure in Figure~\ref{fig:synthetic_results} compares the policy learning effectiveness with varying data sizes, numbers of unique actions, the temperature parameters of the logging policy, and the interaction parameters, where the default parameters are $n$ = 1000, $|\calA| = 5$, $L = 3$, $\tau=1.0$, and $\lambda = 1.0$, respectively.

First, Figure \ref{fig:training_data}, which varies the training data size $n$ from 250 to 2000, shows
R-POD (tuning) performs consistently better than the baseline methods across various data sizes. The advantage of R-POD over the baselines becomes particularly large when the data size is small, suggesting that R-POD effectively achieves a substantial reduction in variance regarding policy gradient estimation to enable a more data-efficient OPL for ranking policies. It would also be interesting to see that R-POD (tuning) performs competitively compared to R-POD (best), particularly when the training data size is large, even though R-POD (best) always performs even better than R-POD (tuning) leveraging its unfair access to the ground-truth policy value to identify the optimal value of $k$.

Next, when varying the numbers of unique actions $|\calA|$ from 4 to 8 (this varies the number of unique rankings from 24 to 336) in Figure \ref{fig:unique_action}, we observe that R-POD outperforms the baseline methods in all situations. This suggests that R-POD can perform satisfactorily even when the number of candidate rankings grows within the evaluated range.

In addition, Figure \ref{fig:behavior_tau} demonstrates that the effectiveness of the policy-based methods such as IPS- and DR-PG worsen for small $\tau$. This is because the ranking-level importance weight is likely to be large as the logging policy becomes close to deterministic with small $\tau$. It is appealing to see that R-POD (tuning) performs much better than the policy-based methods particularly for small $\tau$ by significantly reducing the variance by its top-$k$ importance weighting.
\begin{figure*}[h]
  \centering
  \includegraphics[width=0.8\textwidth]{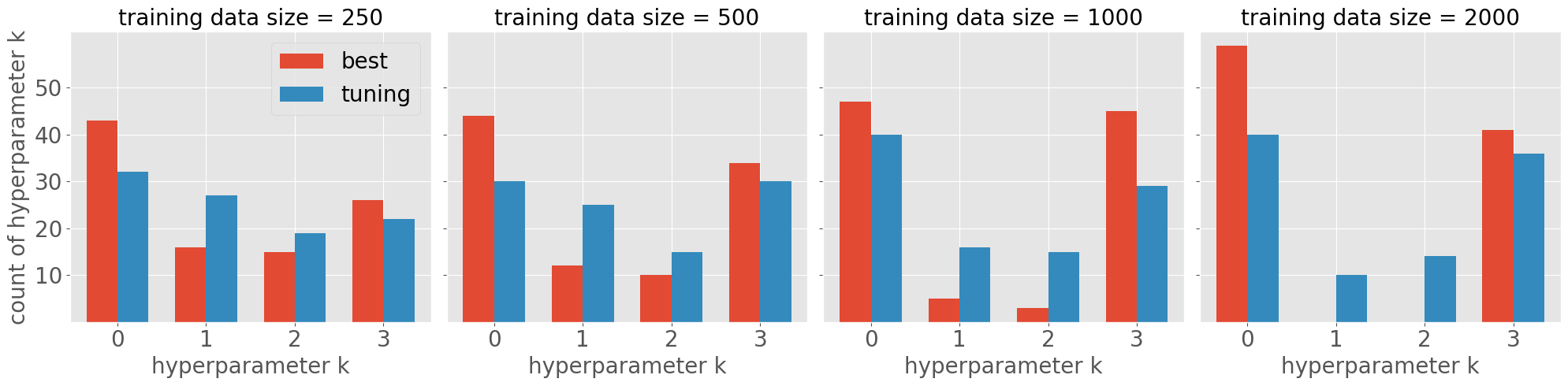}
  \vspace{-3mm}
  \caption{Histograms of hyperparameter $k$ selected by R-POD (best) and R-POD (tuning) over 100 simulations.}\label{fig:topk_count}
\end{figure*}

Finally, Figure~\ref{fig:probability_affected} shows RIPS-PG and IIPS-PG degrade in performance with larger violations of the cascade and independence assumptions (larger $\lambda$) because they ignore ranking interactions. In contrast, R-POD is much more robust to the violations of those assumptions because it unbiasedly estimates the rewards of top-$k$ action via top-$k$ importance weighting without assumptions and also considers the interaction effect from lower positions using the regression model.

Next, we provide ablation results about the selection of the hyperparameter $k$ for R-POD.
Figure \ref{fig:topk_count} reports the number of trials in which each value of hyperparameter $k$ is selected by R-POD (best) and R-POD (tuning). 
The figure demonstrates that the data-driven hyperparameter selection often aligns with the oracle selection of R-POD (best) in many cases, even though R-POD (tuning) does not perfectly match the choice of R-POD (best). Moreover, an interesting observation is that while the choice is polarized to $k=0$ (where R-POD reduces to Reg-based) and $k=3$ (where R-POD reduces to Policy-based) with large data sizes, using intermediate values such as $k=1,2$ can be effective when the data size is small. 
This indicates that, particularly when Reg-based suffers from reward estimation errors and policy-based methods produce high variance with less data, mixing these approaches with R-POD becomes more superior and enables a robust and efficient OPL of ranking policies.

\begin{table*}[t]
\centering
\caption{Comparison of the OPL methods on Yahoo! LETOR (DCG@3) and MSLR-WEB10K (DCG@5). $f_0(x, a)$ in the logging policy uses constant value (uniform) or random forest regression. Values outside and inside the parenthesis are the mean and the standard deviation based on 5 random seeds, respectively.}
\vspace{-2mm}
\renewcommand{\arraystretch}{1.2}
\resizebox{0.68\textwidth}{!}{ 
\begin{tabular}{lcccc} 
\hline
& \multicolumn{2}{c}{\underline{\textbf{Yahoo! LETOR (DCG@3)}}} & \multicolumn{2}{c}{\underline{\textbf{MSLR-WEB10K (DCG@5)}}} \\
& \multicolumn{2}{c}{Logging Policy} & \multicolumn{2}{c}{Logging Policy} \\
\cline{2-3} \cline{4-5}
OPL methods & uniform & random forest & uniform & random forest \\
\hline
Reg-based & 1.660 (0.154) & 1.613 (0.100) & \bf{2.318 (0.237)} & 2.256 (0.269) \\
IPS-PG & 1.568 (0.090) & 1.648 (0.071) & 1.885 (0.060) & 2.016 (0.066) \\
RIPS-PG & 1.576 (0.092) & 1.625 (0.074) & 2.043 (0.041) & 2.078 (0.065) \\
IIPS-PG & 1.602 (0.047) & 1.571 (0.081) & 2.240 (0.077) & 2.227 (0.155) \\
DR-PG & 1.599 (0.103) & 1.672 (0.033) & 1.893 (0.078) & 2.227 (0.100) \\
R-POD (tuning, \textbf{large} err.) & 1.651 (0.068) & 1.651 (0.041) & 2.274 (0.159) & 2.229 (0.175) \\
R-POD (tuning, \textbf{small} err.) & \bf{1.679 (0.093)} & \bf{1.682 (0.055)} & 2.317 (0.214) & \bf{2.280 (0.088)} \\
R-POD (best) & \bf{1.694 (0.091)} & \bf{1.693 (0.046)} & \bf{2.335 (0.213)} & \bf{2.337 (0.122)} \\
\hline
\end{tabular}
} \label{table:comparison}
\vspace{-3mm}
\end{table*}

\subsection{Real-World Data}
Next, we conduct real-world experiments on two ranking datasets, namely the Microsoft Learning to Rank Challenge dataset (MSLR-WEB10K)~\cite{qin2013introducing} and Yahoo! learning to rank challenge dataset (Yahoo! LETOR)~\cite{chapelle2011yahoo}.  MSLR-WEB10K has 124 documents per query and Yahoo! LETOR has 24 documents per query on average, and we randomly sample 100 documents per query from MSLR-WEB10K and 10 documents from Yahoo! LETOR.

These datasets contain 5-level relevance scores $\textit{rel}(\bx, a) \in \{0, ..., 4\}$
for all of their query$(\bx)$-document$(a)$ pairs. To define the expected reward function for each position and ranking, we first define the base reward function for each unique item as $\tilde{q}_l(\bx, a_l) = \textit{rel}(\bx, a_l)/4 + \eta_{a_l}$ where $\eta_{a_l}$ is a noise parameter sampled separately for each $a_l$ from a normal distribution whose mean is 0 and standard deviation is 0.05. We also synthesize $F (\bx, \ba)$ to introduce interaction effects in a ranking similarly to the synthetic experiment and define the position-wise expected reward function $q_l(\bx, \ba)$ as in Eq.~\eqref{eq:q_func}.
Then, we sample the reward for each position $l$ from a normal distribution with mean $q_l(\bx, \ba)$ and standard deviation $\sigma=0.05$. We use the Plackett-Luce logging policy defined in Eq.~\eqref{eq:plackett-luce}, the same logging policy as used in the synthetic experiment. However, in the real-world experiment, $f_0(\bx, a)$ in the logging policy definition is either a constant value (uniform) or a regression model (random forest regression), which is trained with 10 \% of the training data. Note that we set $L=5$ for MSLR-WEB10K and $L=3$ for Yahoo! LETOR.
R-POD (tuning) tunes its hyperparameter $k$ by a noise-added ground-truth policy value where the noise (or estimation error) is sampled from a uniform distribution of range $(-\Delta_{\text{max}}, \Delta_{\text{max}})$, and $|\Delta_{\text{max}}|$ is either $V(\pi_{\theta})/20$ (\textbf{small} estimation error) or $V(\pi_{\theta})/10$ (\textbf{large} estimation error), varying the accuracy of the tuning of $k$.

\paragraph{\textbf{Results}}
Table \ref{table:comparison} reports the real-world experiment results.
The results demonstrate that R-POD generally outperforms the baseline methods across different datasets and logging policies. Specifically, we observe that R-POD significantly outperforms RIPS-PG and IIPS-PG, even with the tuned hyperparameter $k$ under a \textbf{large} estimation error. Moreover, when comparing R-POD with Reg-based methods, we find that Reg-based is competitive with R-POD only when the reward regressor is accurate (as in MSLR-WEB10K), but it underperforms R-POD when the regression is inaccurate (as in Yahoo! LETOR). These results suggest that R-POD is more robust to regression error than Reg-based, which aligns with our theoretical analysis.
Furthermore, R-POD performs much better than DR-PG and IPS-PG when the action space is large (MSLR-WEB10K), due to its substantial variance reduction in policy gradient estimation without introducing significant bias. We also observe that R-POD (best) consistently achieves the best results, indicating the even greater potential of our R-POD algorithm on real-world datasets, especially with an improved procedure for tuning its key hyperparameter $k$.

\section{related work}
\paragraph{\textbf{Off-Policy Evaluation for Ranking Policies}}
In ranking settings, the action space consists of permutations of unique items, which often causes severe variance in off-policy estimation~\cite{li2018offline,mcinerney2020counterfactual,tanaka2026offpolicy}. Existing estimators reduce this variance by exploiting sub-rewards and assumptions on how users examine ranked lists, often motivated by click models~\cite{guo2009efficient,chuklin2015click}.
IIPS~\cite{li2018offline} and RIPS~\cite{mcinerney2020counterfactual} reduce variance using the independence and cascade assumptions, respectively, but can introduce bias when these assumptions are violated~\cite{mcinerney2020counterfactual,kiyohara2022doubly}. Cascade-DR~\cite{kiyohara2022doubly} incorporates a baseline estimator as a control variate and further reduces variance, while still depending on the cascade view of user behavior.

These estimators impose a single user-behavior assumption across all users, even though user behavior is diverse in practice~\cite{borisov2016neural,xu2012incorporating}. Adaptive IPS (AIPS)~\cite{kiyohara2023off} addresses this issue by choosing adaptive importance weights depending on each user, and provides strong empirical improvements under heterogeneous user behavior. Another line of work views ranking as an episodic reinforcement learning problem by modeling click behavior as an MDP, which enables the use of offline RL estimators~\cite{zhang2023unified}.

Our proposed OPL method is motivated by the observation that higher-ranked items usually dominate the expected reward of a ranking. While this resembles the cascade assumption, our gradient estimator also uses a reward regression model to account for the effects of lower-ranked actions. As a result, it can exploit the structure of rankings without introducing additional bias from a fixed user-behavior assumption. Practical policy-learning methods for top-$k$ recommendation also use behavior assumptions and regularization~\cite{liu2022practical}; since these methods are variants of IIPS and RIPS, our experiments compare R-POD with IIPS and RIPS as representative baselines.

\paragraph{\textbf{Off-Policy Learning}}
Off-policy learning (OPL) for contextual bandits aims to learn a new policy using only logged data \cite{dudik2011doubly,swaminathan2015counterfactual,saito2021counterfactual,tanaka2026off}. Existing methods are commonly divided into regression-based and policy-based approaches. Regression-based methods estimate the expected reward function with supervised learning and select actions with high predicted rewards, but their performance can be biased when the reward model is inaccurate. Policy-based methods directly optimize a parameterized policy using gradient estimates from logged data, typically through importance weighting. These estimators are unbiased under a full-support condition, but this condition is difficult to satisfy in large action spaces. Even when it holds, the resulting importance weights can be extremely large, leading to high variance~\cite{sachdeva2020off,felicioni2022off}. Regularized OPL methods mitigate this issue by penalizing deviations from the logging policy~\cite{jeunen2021pessimistic,ma2019imitation,liang2023local}, but they may also limit improvement because the learned policy remains close to the logging policy.

POTEC~\cite{saito2025potec} addresses large action spaces by decomposing policy optimization into two stages using the Conjunct Effect Model~\cite{saito2023off}. It first selects a promising action cluster with a policy-based method and then chooses an action within that cluster using a regression-based method. Since importance weighting is applied over clusters rather than the original action space, POTEC can substantially reduce variance while retaining favorable bias properties under a local correctness condition. However, its effectiveness depends on constructing useful action clusters.

Our R-POD algorithm is inspired by POTEC in combining policy-based and regression-based optimization, but it decomposes a ranking policy using the intrinsic structure of ranked lists rather than external clusters. Consequently, R-POD does not require learning or tuning a clustering method; it only requires choosing the top-$k$ decomposition point. This makes the bias--variance trade-off easier to control and directly leverages the structure of ranking problems.

\section{Conclusion}
This work studies off-policy learning (OPL) for ranking policies. Existing methods often fail due to substantial bias and variance. To facilitate more effective OPL for ranking, we develop a novel algorithm called R-POD. R-POD optimizes the first-stage policy, responsible for selecting the top-$k$ actions in a ranking, through a new policy gradient estimator. This estimator is unbiased under a relaxed condition about reward regression and exhibits lower variance compared to existing gradient estimators. The second-stage policy, responsible for selecting the bottom actions, is learned via a reward regression. This component of our algorithm is more robust to reward modeling errors than traditional regression-based methods, because the policy gradient part of the algorithm already unbiasedly estimates the value of top-$k$ actions. Empirical evaluations demonstrate the effectiveness of R-POD in optimizing ranking policies, particularly in challenging situations such as with small sample sizes and large ranking spaces.

\appendix

\section{Optimizing the regression model via a two-step procedure}
We can optimize the regression model via a two-step procedure instead of the one-step procedure in Eq.\ref{eq:one_step_regression}. We describe how to implement and use this two-step procedure in the R-POD algorithm.\\
\label{two-step procedure}
\subsection{Two-step Procedure}

Proposition~\ref{prop:rpod-variance} suggests that, in terms of variance minimization, we should optimize the regression model in a way that minimizes $|\Delta_{q,\hat{f}}(\bx,\ba)|$ compared to minimizing $|\Delta_{q} (\bx,\ba,\bb) - \Delta_{\hat{f}} (\bx,\ba,\bb)|$ for the bias. Therefore, based on the theoretical observations, we should ideally optimize the regression model via the following two-step procedure in order to optimize the bias and variance of the R-POD gradient estimator.
\\

\textbf{1. Bias Minimization Step}: Optimize a pairwise regression function $\hat{h}_{\phi}$, parameterized by $\phi$, to estimate the relative value differences of ranking sharing the same top-$k$ actions.
\begin{align}
    \label{eq:b_mini_step}
    \min_{\phi} \sum_{(\bx,\ba,\bb,\br_{\ba},\br_{\bb}) \in \mathcal{D}_{pair}}
    \ell_h \left(\br_{\ba}-\br_{\bb} , h_{\phi}(\bx,\ba)
    -h_{\phi}(\bx,\bb)\right)
\end{align}
\textbf{2. Variance Minimization Step}: Optimize $\hat{g}_{\omega}$, parameterized by $\omega$, to minimize $\Delta_{q,\hat{f}}(\bx,\ba)$ given $\hat{f} = \hat{g}_{\omega} + \hat{h}_{\phi}$.
\begin{align}
    \label{eq:v_mini step}
    \min_{\omega} \sum_{(\bx,\ba,\br) \in \mathcal{D}}
     \ell_g \left(\br , \hat{g}_{\omega}(\bx,\ba) + \hat{h}_{\phi}(\bx,\ba)\right)
\end{align}
$\ell_h$, $\ell_g$ : $\mathbb{R} \times{\mathbb{R}} \to \mathbb{R}_{\geq 0}$ are some appropriate loss functions such as squared loss. Note that $\mathcal{D}_{pair}$ is a dataset augmented for performing pairwise regression to minimize the bias of the R-POD gradient estimator, which is defined as
\begin{align}
    & \mathcal{D}_{pair} := \Bigg\{(\bx,\ba,\bb,\br_{\ba},\br_{\bb}) \mid 
    \begin{array}{l} 
      (\bx_{\ba},\ba,\br_{\ba}), (\bx_{\bb},\bb,\br_{\bb}) \in  \mathcal{D}\\
     \bx = \bx_{\ba} = \bx_{\bb}, \ba_{1:k} = \bb_{1:k} 
    \end{array} 
    \Bigg\}.
\end{align}
As suggested in our analysis, $\hat{h}_{\phi}(\bx,\ba)$ characterizes the bias of the R-POD gradient estimator, so the first step focuses on minimizing its bias by optimizing it towards accurately estimating the relative reward differences. The second step then aims for variance minimization by optimizing $\hat{g}_{\omega}$. Since the bias of the R-POD gradient estimator does not depend on the top-$k$ actions (as in Theorem~\ref{thm:bias_of_rpod-pg}), the second step minimizes its variance without affecting its bias. After performing the two-step regression procedure, we can construct a regression model as $\hat{f}_{\omega,\phi}(\bx,\ba) = \hat{g}_{\omega}(\bx,\ba_{1:k}) + \hat{h}_{\phi}(\bx,\ba)$ that is to be used as a part of the R-POD gradient estimator. 

Note that we do not expect that the two-step procedure described above is always feasible and practical due to its need for pairwise data and implementation costs. If the two-step procedure is impractical due to a lack of sufficient pairwise data, we can employ a regression for the expected absolute reward in Eq. \eqref{eq:one_step_regression}. \\

\subsection{Optimizing the Second-stage Policy $\pi_{\phi}^{2nd}$}

The objective of the second-stage policy $\pi_{\phi}^{2nd}$ is to identify the bottom actions to optimize the value given the top-$k$ actions chosen by the first-stage policy $\pi_{\theta}^{1st}$. Specifically, the second-stage policy should rank the bottom actions that yield the highest value among rankings containing the same top-$k$ actions. Fortunately, from the two-step regression procedure mentioned earlier, we have already obtained $\hat{h}_{\phi}$ aiming to preserve the relative value difference of rankings. This allows us to readily define the second-stage policy based on the pairwise regressor $\hat{h}_{\phi}$ as follows.
\begin{align}
    \label{2step train pi_2}
    \pi_{\phi}^{2nd}(\ba_{k+1:L}|\bx,\ba_{1:k}) \coloneq
   \begin{cases}
      1\quad (\ba = \argmax_{\ba':\ba'_{1:k}=\ba_{1:k}} \hat{h}_{\phi}(\bx,\ba'))\\
      0\quad (otherwise)
   \end{cases}
\end{align}

The second-stage policy is constructed solely based on pairwise regressor $\hat{h}_{\phi}$. Although optimized by the regression-based approach, it only needs to learn the relative value difference of rankings, which is easier than precisely learning the global reward function. Thus, the second-stage policy of our method is expected to produce lower bias compared to regression-based approaches.\\

\section{Omitted Proofs}
\begingroup
\small
\allowdisplaybreaks
Here, we provide the derivations and proofs omitted in the main text.
\subsection{Derivation of Eq. \eqref{eq:truePG}}
\label{derivation of overall gra}
\begin{proof}
    We derive the overall policy gradient in Eq. \eqref{eq:truePG}.
    \begin{align}
        &\nabla_{\theta}V(\pi_{\theta,\phi}^{overall})\notag\\
        &=\nabla_{\theta} \mathbb{E}_{p(\bx)\overallpolicy}\left[\sum_{l=1}^{L} \alpha_{l}q_l(\bx,\ba)\right]\notag\\
        &=\nabla_{\theta} \mathbb{E}_{p(\bx)\firstpolicy\secondpolicy}
        \left[\sum_{l=1}^{L} \alpha_{l}q_l(\bx,\ba)\right]\notag\\
        &=\mathbb{E}_{p(\bx)}
        \left[\sum_{\ba_{1:k}}\sum_{\ba_{k+1:L}}\nabla_{\theta} \firstpolicy\secondpolicy
        \sum_{l=1}^{L} \alpha_{l}q_l(\bx,\ba)\right]\notag\\
        &=\mathbb{E}_{p(\bx)\firstpolicy}
        \left[s_\theta(\bx,\ba_{1:k}) \mathbb{E}_{\secondpolicy}
        \left[\sum_{l=1}^{L} \alpha_{l}q_l(\bx,\ba)\right]\right]\notag\\
        &=\mathbb{E}_{p(\bx)\firstpolicy}
        \left[q^{\pi_{\phi}^{2nd}}(\bx,\ba_{1:k}) s_\theta(\bx,\ba_{1:k})\right]
    \end{align}
    where we use $\mathbb{E}_{\secondpolicy}\left[\sum_{l=1}^{L} \alpha_{l}q_l(\bx,\ba)\right] 
    \coloneqq q^{\pi_{\phi}^{2nd}}(\bx,\ba_{1:k})$ and 
    $\nabla_{\theta}\log{\firstpolicy} \coloneqq s_\theta(\bx,\ba_{1:k})$.
\end{proof}

\subsection{Derivation of Eq. \eqref{eq:bias_of_rpod-pg}}
\begin{proof}
    We derive the bias of R-POD in Eq. \eqref{eq:bias_of_rpod-pg}.
    \begin{align}
        &Bias(\nabla_\theta\widehat{V}_{RPOD}(\pi_{\theta,\phi}^{overall}; \mathcal{D}))\notag\\
        &=\mathbb{E}_{p(\bx)\loggingfirstpolicy\loggingsecondpolicy}\\
        &\qquad \qquad \left[w(\bx,\ba_{1:k})
        \left(\sum_{l=1}^{L}\alpha_{l}q_{l}(\bx,\ba) - \hat{f}(\bx,\ba)\right)s_\theta(\bx,\ba_{1:k})\right]\notag\\
        &\quad -\mathbb{E}_{p(\bx)\pi_{\theta}^{1st}(\ba_{1:k}|\bx)}
        \left[\left(q^{\pi_{\phi}^{2nd}}(\bx,\ba_{1:k})
        -\hat{f}^{\pi_{\phi}^{2nd}}(\bx,\ba_{1:k})\right)
        \times s_\theta(\bx,\ba_{1:k})\right]\notag\\    
        &=\mathbb{E}_{p(\bx)\loggingfirstpolicy}
        \left[w(\bx,\ba_{1:k})s_\theta(\bx,\ba_{1:k})
        \sum_{\ba_{k+1:L}}\loggingsecondpolicy \right.\notag\\
        &\qquad\qquad\left.
        \times\left(\sum_{l=1}^{L}\alpha_{l}q_{l}(\bx,\ba) - \hat{f}(\bx,\ba)\right)\right]\notag\\
        &\quad-\mathbb{E}_{p(\bx)}\left[
        \sum_{\ba_{1:k}}\firstpolicy
        \left(q^{\pi_{\phi}^{2nd}}(\bx,\ba_{1:k})
        -\hat{f}^{\pi_{\phi}^{2nd}}(\bx,\ba_{1:k})\right) s_\theta(\bx,\ba_{1:k})\right]\notag\\
        &=\mathbb{E}_{p(\bx)\loggingfirstpolicyc}\left[s_\theta(\bx,\bc_{1:k})
        \sum_{\ba:\ba_{1:k}=\bc_{1:k}} w(\bx,\ba) 
        \pi_{0}^{2nd}(\ba_{k+1:L}|\boldsymbol{x},\bc_{1:k}) \right.\notag\\
        &\qquad\qquad\left.
        \times\sum_{\bb:\bb_{1:k}=\bc_{1:k}}
        \pi_{0}^{2nd}(\bb_{k+1:L}|\boldsymbol{x},\bc_{1:k})\left(\sum_{l=1}^{L}\alpha_{l}q_{l}(\bx,\bb) - \hat{f}(\bx,\bb)\right)\right]\notag\\
        &\quad -\mathbb{E}_{p(\bx)\loggingfirstpolicyc}
        \left[ w(\bx,\bc_{1:k}) s_\theta(\bx,\bc_{1:k})
        \sum_{\ba:\ba_{1:k}=\bc_{1:k}}\right.\notag\\
        &\qquad\qquad\left.
        \times \pi_{\phi}^{2nd}(\ba_{k+1:L}|\boldsymbol{x},\bc_{1:k})\left(\sum_{l=1}^{L}\alpha_{l}q_{l}(\bx,\ba) - \hat{f}(\bx,\ba)\right)\right] \label{expectation q}\\
        &=\mathbb{E}_{p(\bx)\loggingfirstpolicyc}\left[s_\theta(\bx,\bc_{1:k})
        \sum_{\ba:\ba_{1:k}=\bc_{1:k}} w(\bx,\ba) 
        \pi_{0}^{2nd}(\ba_{k+1:L}|\boldsymbol{x},\bc_{1:k}) \right.\notag\\
        &\qquad\qquad\left.
        \times\sum_{\bb:\bb_{1:k}=\bc_{1:k}}
        \pi_{0}^{2nd}(\bb_{k+1:L}|\boldsymbol{x},\bc_{1:k})
        \Delta_{q,\hat{f}}(\bx,\bb)\right]\notag\\
        &\quad -\mathbb{E}_{p(\bx)\loggingfirstpolicyc}
        \left[ s_\theta(\bx,\bc_{1:k})
        \sum_{\ba:\ba_{1:k}=\bc_{1:k}} w(\bx,\bc_{1:k})\right.\notag\\
        &\qquad\qquad\left.
        \times \pi_{\phi}^{2nd}(\ba_{k+1:L}|\boldsymbol{x},\bc_{1:k})
        \Delta_{q,\hat{f}}(\bx,\bb)\right]\notag\\
        &=\mathbb{E}_{p(\bx)\loggingfirstpolicyc}\bigl[s_\theta(\bx,\bc_{1:k})
        \sum_{\ba:\ba_{1:k}=\bc_{1:k}} w(\bx,\ba)
        \pi_{0}^{2nd}(\ba_{k+1:L}|\boldsymbol{x},\bc_{1:k}) \notag\\
        &\qquad\qquad
        \times\biggl(\biggl(\sum_{\bb:\bb_{1:k}=\bc_{1:k}}
        \pi_{0}^{2nd}(\bb_{k+1:L}|\boldsymbol{x},\bc_{1:k})\Delta_{q,\hat{f}}(\bx,\bb)\biggr) - \Delta_{q,\hat{f}}(\bx,\ba)\biggr)\bigr]\notag
    \end{align}
    where $\Delta_{q,\hat{f}}(\bx,\ba) \coloneqq \sum_{l=1}^{L}\alpha_{l}q_{l}(\bx,\ba) - \hat{f}(\bx,\ba)$.\\
    We use $q^{\pi_{\phi}^{2nd}}(\bx,\ba_{1:k})=\mathbb{E}_{\secondpolicy}\left[\sum_{l=1}^{L} \alpha_{l}q_l(\bx,\ba)\right]$ and\\
    $\hat{f}^{\pi_{\phi}^{2nd}}(\bx,\ba_{1:k})
    = \mathbb{E}_{\secondpolicy}\left[\hat{f}(\bx,\ba)\right]$ in Eq. \eqref{expectation q}.\\
    We use Lemma B.1 of \cite{saito2022off} and then get the following.
    \begin{align}
        &\mathbb{E}_{p(\bx)\loggingfirstpolicyc}
        \left[s_\theta(\bx,\bc_{1:k})
        \sum_{\substack{\ba<\bb:\\ \ba_{1:k}=\bb_{1:k}=\bc_{1:k}}}
        \pi_{0}^{2nd}(\ba_{k+1:L}|\bx,\boldsymbol{c}_{1:k})\right.\notag\\
        &\qquad\qquad\left.
        \times\pi_{0}^{2nd}(\bb_{k+1:L}|\bx,\boldsymbol{c}_{1:k})
        \left(\Delta_{q,\hat{f}}(\bx,\ba) - \Delta_{q,\hat{f}}(\bx,\bb)\right)\right.\notag\\
        &\qquad\qquad\left.
        \times\left(w(\bx,\bb) - w(\bx,\ba)\right)\right]\notag\\
        &= \mathbb{E}_{p(\bx)\pi_{0}^{1st}(\bc_{1:k}|\bx)}
        \bigg[\sum_{\substack{\ba<\bb:\\ \ba_{1:k}=\bb_{1:k}=\bc_{1:k}}}
        \pi_{0}^{2nd}(\ba_{k+1:L}|\bx,\boldsymbol{c}_{1:k})\notag\\
        &\qquad\qquad\times
        \pi_{0}^{2nd}(\bb_{k+1:L}|\bx,\boldsymbol{c}_{1:k})\notag\\
        &\qquad\qquad\times
        (\Delta_{q}(\bx,\ba,\bb)-\Delta_{\hat{f}}(\bx,\ba,\bb))
        s_\theta(\bx,\boldsymbol{c}_{1:k})\notag\\
        &\qquad\qquad\times\left(\frac{\pi_{\theta}(\bb|\bx,\boldsymbol{c}_{1:k})}
        {\pi_{0}(\bb|\bx,\boldsymbol{c}_{1:k})}
        -\frac{\pi_{\theta}(\ba|\bx,\boldsymbol{c}_{1:k})}
        {\pi_{0}(\ba|\bx,\boldsymbol{c}_{1:k})}\right)\bigg],\notag
    \end{align}
    where we use $\Delta_{q,\hat{f}}(\bx,\ba) - \Delta_{q,\hat{f}}(\bx,\bb) \Rightarrow$\\
    $\Delta_{q}(\bx,\ba,\bb)-\Delta_{\hat{f}}(\bx,\ba,\bb)$.
\end{proof}

\subsection{Derivation of Proposition \ref{prop:rpod-variance}}
\begin{proof}
    We derive the variance of R-POD under Condition \ref{eq:full_top-k_support}
    and Condition \ref{ass:pair_correct} by applying the law of total variance several times.
    \begin{align}
        &n\mathbb{V}_{\mathcal{D}}(\nabla_\theta\widehat{V}_{RPOD}^{(j)}
        (\pi_{\theta,\phi}^{overall}; \mathcal{D}))\notag\\
        &=\mathbb{E}_{p(\bx)\pi_0(\ba|\bx)}
        \left[\mathbb{V}_{p(\vecr|\bx,\ba)}\left[
         w(\bx,\ba_{1:k})
        \left(\sum_{l=1}^{L}\alpha_{l}r_{l} - \hat{f}(\bx,\ba)\right)\right.\right.\notag\\
        &\qquad\qquad\left.\left.
        \times s_{\theta}^{(j)}(\bx,\ba_{1:k})
        + \mathbb{E}_{\pi_{\theta}^{1st}(\ba_{1:k}|\bx)}
        \left[\hat{f}^{\pi_{\phi}^{2nd}}(\bx,\ba_{1:k})s_{\theta}^{(j)}(\bx,\ba_{1:k})\right]\right]\right]\notag\\
        &\quad + \mathbb{V}_{p(\bx)\pi_0(\ba|\bx)}
        \left[\mathbb{E}_{p(\vecr|\bx,\ba)}\left[
         w(\bx,\ba_{1:k})
        \left(\sum_{l=1}^{L}\alpha_{l}r_{l} - \hat{f}(\bx,\ba)\right)\right.\right.\notag\\
        &\qquad\qquad\left.\left.
        \times s_{\theta}^{(j)}(\bx,\ba_{1:k})
        + \mathbb{E}_{\pi_{\theta}^{1st}(\ba_{1:k}|\bx)}
        \left[\hat{f}^{\pi_{\phi}^{2nd}}(\bx,\ba_{1:k})s_{\theta}^{(j)}(\bx,\ba_{1:k})\right]\right]\right]\notag\\
        &=\mathbb{E}_{p(\bx)\pi_0(\ba|\bx)}\left[
         \left(w(\bx,\ba_{1:k})s_{\theta}^{(j)}(\bx,\ba_{1:k})\right)^2
         \sigma^2(\bx,\ba)\right]\notag\\
        &\quad + \mathbb{E}_{p(\bx)}
        \left[\mathbb{V}_{\pi_0(\ba|\bx)}\left[
         w(\bx,\ba_{1:k})
         \left(\sum_{l=1}^{L}\alpha_{l}q_{l}(\bx,\ba)-\hat{f}(\bx,\ba)\right)\right.\right.\notag\\
        &\qquad\qquad\left.\left.
        \times s_{\theta}^{(j)}(\bx,\ba_{1:k})
        + \mathbb{E}_{\pi_{\theta}^{1st}(\ba_{1:k}|\bx)}
        \left[\hat{f}^{\pi_{\phi}^{2nd}}(\bx,\ba_{1:k})s_{\theta}^{(j)}(\bx,\ba_{1:k})\right]\right]\right]\notag\\
        &\quad + \mathbb{V}_{p(\bx)}
        \left[\mathbb{E}_{\pi_0(\ba|\bx)}\left[
         w(\bx,\ba_{1:k})
         \left(\sum_{l=1}^{L}\alpha_{l}q_{l}(\bx,\ba)-\hat{f}(\bx,\ba)\right)\right.\right.\notag\\
        &\qquad\qquad\left.\left.
        \times s_{\theta}^{(j)}(\bx,\ba_{1:k})
        + \mathbb{E}_{\pi_{\theta}^{1st}(\ba_{1:k}|\bx)}
        \left[\hat{f}^{\pi_{\phi}^{2nd}}(\bx,\ba_{1:k})s_{\theta}^{(j)}(\bx,\ba_{1:k})\right]\right]\right]\notag\\
        &=\mathbb{E}_{p(\bx)\pi_0(\ba|\bx)}\left[
         \left(w(\bx,\ba_{1:k})s_{\theta}^{(j)}(\bx,\ba_{1:k})\right)^2
         \sigma^2(\bx,\ba)\right]\notag\\
        &\quad
        + \mathbb{E}_{p(\bx)}\left[\mathbb{V}_{\pi_0(\ba|\bx)}\left[
         w(\bx,\ba_{1:k})\Delta_{q,\hat{f}}(\bx,\ba)
         s_{\theta}^{(j)}(\bx,\ba_{1:k})\right]\right]\notag\\
        &\quad + \mathbb{V}_{p(\bx)}
        \left[\mathbb{E}_{\firstpolicy}\left[\mE_{ \loggingsecondpolicy}\left[
         g(\bx,\ba_{1:k}) s_{\theta}^{(j)}(\bx,\ba_{1:k})\right]\right]\right.\notag\\
        &\qquad\qquad\left.
        + \mathbb{E}_{\pi_{\theta}^{1st}(\ba_{1:k}|\bx)}
        \left[\mE_{\pi_{\phi}^{2nd}} 
        \left[\left(q(\bx,\ba)-g(\bx,\ba_{1:k})\right)s_{\theta}^{(j)}(\bx,\ba_{1:k})\right]\right]\right]
        \label{use local correctness}\\
        &=\mathbb{E}_{p(\bx)\pi_0(\ba|\bx)}\left[
         \left(w(\bx,\ba_{1:k})s_{\theta}^{(j)}(\bx,\ba_{1:k})\right)^2
         \sigma^2(\bx,\ba)\right]\notag\\
        &\quad
        + \mathbb{E}_{p(\bx)}\left[\mathbb{V}_{\pi_0(\ba|\bx)}\left[
         w(\bx,\ba_{1:k})\Delta_{q,\hat{f}}(\bx,\ba)
         s_{\theta}^{(j)}(\bx,\ba_{1:k})\right]\right]\notag\\
        &\quad + \mathbb{V}_{p(\bx)}
        \left[\mathbb{E}_{\pi_{\theta}^{1st}(\ba_{1:k}|\bx)}\left[
         q^{\pi_{\phi}^{2nd}}(\bx,\ba_{1:k})
         s_{\theta}^{(j)}(\bx,\ba_{1:k})\right]\right],\notag
    \end{align}
    where we use the conditional pairwise correctness condition as
    $g(\bx,\ba_{1:k}) = q(\bx,\ba) - \hat{f}(\bx,\ba)$ in Eq. \eqref{use local correctness}.
    $s_{\theta}^{(j)}(\bx,\ba_{1:k})$ is the \textit{j}-th dimension of the score function.
\end{proof}
\endgroup

\begin{acks}
Haruka Kiyohara is supported in part by the Funai Overseas Scholarship, Quad Fellowship, and LinkedIn PhD Award.
\end{acks}

\section*{GenAI Usage Disclosure}
The authors used generative AI tools to assist with code development and debugging and to improve the language and clarity of the manuscript. The authors reviewed and verified all AI-assisted code and text and take full responsibility for the content and results of this work.

\bibliographystyle{ACM-Reference-Format}
\bibliography{ref}

@inproceedings{chapelle2011yahoo,
  title={Yahoo! learning to rank challenge overview},
  author={Chapelle, Olivier and Chang, Yi},
  booktitle={Proceedings of the learning to rank challenge},
  pages={1--24},
  year={2011},
  organization={PMLR}
}

@article{qin2013introducing,
  title={Introducing LETOR 4.0 datasets},
  author={Qin, Tao and Liu, Tie-Yan},
  journal={arXiv preprint arXiv:1306.2597},
  year={2013}
}

@article{plackett1975analysis,
  title={The analysis of permutations},
  author={Plackett, Robin L},
  journal={Journal of the Royal Statistical Society Series C: Applied Statistics},
  volume={24},
  number={2},
  pages={193--202},
  year={1975},
  publisher={Oxford University Press}
}

@inproceedings{saito2022off,
  title={Off-Policy Evaluation for Large Action Spaces via Embeddings},
  author={Saito, Yuta and Joachims, Thorsten},
  booktitle={International Conference on Machine Learning},
  pages={19089--19122},
  year={2022},
  organization={PMLR}
}

@inproceedings{jeunen2021pessimistic,
  title={Pessimistic reward models for off-policy learning in recommendation},
  author={Jeunen, Olivier and Goethals, Bart},
  booktitle={Proceedings of the 15th ACM Conference on Recommender Systems},
  pages={63--74},
  year={2021}
}

@inproceedings{ma2019imitation,
  title={Imitation-regularized offline learning},
  author={Ma, Yifei and Wang, Yu-Xiang and Narayanaswamy, Balakrishnan},
  booktitle={The 22nd International Conference on Artificial Intelligence and Statistics},
  pages={2956--2965},
  year={2019},
  organization={PMLR}
}

@inproceedings{saito2021counterfactual,
    author = {Saito, Yuta and Joachims, Thorsten},
    title = {Counterfactual Learning and Evaluation for Recommender Systems: Foundations, Implementations, and Recent Advances},
    year = {2021},
    booktitle = {Proceedings of the 15th ACM Conference on Recommender Systems},
    pages = {828–830},
}

@article{chuklin2015click,
  title={Click Models for Web Search},
  author={Chuklin, Aleksandr and Markov, Ilya and Rijke, Maarten de},
  journal={Synthesis lectures on information concepts, retrieval, and services},
  volume={7},
  number={3},
  pages={1--115},
  year={2015},
  publisher={Morgan \& Claypool Publishers}
}

@inproceedings{guo2009efficient,
  title={Efficient Multiple-Click Models in Web Search},
  author={Guo, Fan and Liu, Chao and Wang, Yi Min},
  booktitle={Proceedings of the 2nd ACM International Conference on Web Search and Data Mining},
  pages={124--131},
  year={2009}
}

@inproceedings{li2018offline,
  title={Offline Evaluation of Ranking Policies with Click Models},
  author={Li, Shuai and Abbasi-Yadkori, Yasin and Kveton, Branislav and Muthukrishnan, S and Vinay, Vishwa and Wen, Zheng},
  booktitle={Proceedings of the 24th ACM SIGKDD International Conference on Knowledge Discovery and Data Mining},
  pages={1685--1694},
  year={2018}
}

@inproceedings{mcinerney2020counterfactual,
  title={Counterfactual Evaluation of Slate Recommendations with Sequential Reward Interactions},
  author={McInerney, James and Brost, Brian and Chandar, Praveen and Mehrotra, Rishabh and Carterette, Benjamin},
  booktitle={Proceedings of the 26th ACM SIGKDD International Conference on Knowledge Discovery and Data Mining},
  pages={1779--1788},
  year={2020}
}

@inproceedings{precup2000eligibility,
  author = {Precup, Doina and Sutton, Richard S. and Singh, Satinder P.},
  title = {Eligibility Traces for Off-Policy Policy Evaluation},
  year = {2000},
  booktitle = {Proceedings of the 17th International Conference on Machine Learning},
  pages = {759–766},
}

@inproceedings{sachdeva2020off,
  title={Off-Policy Bandits with Deficient Support},
  author={Sachdeva, Noveen and Su, Yi and Joachims, Thorsten},
  booktitle={Proceedings of the 26th ACM SIGKDD International Conference on Knowledge Discovery and Data Mining},
  pages={965--975},
  year={2020}
}

@inproceedings{peng2023offline,
  title={Offline Policy Evaluation in Large Action Spaces via Outcome-Oriented Action Grouping},
  author={Peng, Jie and Zou, Hao and Liu, Jiashuo and Li, Shaoming and Jiang, Yibao and Pei, Jian and Cui, Peng},
  booktitle={Proceedings of the ACM Web Conference 2023},
  pages={1220--1230},
  year={2023}
}

@inproceedings{swaminathan2015counterfactual,
  title={Counterfactual risk minimization: Learning from logged bandit feedback},
  author={Swaminathan, Adith and Joachims, Thorsten},
  booktitle={International Conference on Machine Learning},
  pages={814--823},
  year={2015},
  organization={PMLR}
}

@inproceedings{joachims2017unbiased,
  title={Unbiased learning-to-rank with biased feedback},
  author={Joachims, Thorsten and Swaminathan, Adith and Schnabel, Tobias},
  booktitle={Proceedings of the Tenth ACM International Conference on Web Search and Data Mining},
  pages={781--789},
  year={2017}
}

@inproceedings{dudik2011doubly,
  title={Doubly robust policy evaluation and learning},
  author={Dud{\'\i}k, Miroslav and Langford, John and Li, Lihong},
  booktitle={Proceedings of the 28th International Conference on International Conference on Machine Learning},
  pages={1097--1104},
  year={2011}
}

@article{felicioni2022off,
  title={Off-Policy Evaluation with Deficient Support Using Side Information},
  author={Felicioni, Nicolo and Ferrari Dacrema, Maurizio and Restelli, Marcello and Cremonesi, Paolo},
  journal={Advances in Neural Information Processing Systems},
  volume={35},
  year={2022}
}

@inproceedings{kiyohara2022doubly,
  author = {Kiyohara, Haruka and Saito, Yuta and Matsuhiro, Tatsuya and Narita, Yusuke and Shimizu, Nobuyuki and Yamamoto, Yasuo},
  title = {Doubly Robust Off-Policy Evaluation for Ranking Policies under the Cascade Behavior Model},
  booktitle = {Proceedings of the 15th International Conference on Web Search and Data Mining},
  year = {2022},
}

@inproceedings{kiyohara2023off,
  title={Off-Policy Evaluation of Ranking Policies under Diverse User Behavior},
  author={Kiyohara, Haruka and Uehara, Masatoshi and Narita, Yusuke and Shimizu, Nobuyuki and Yamamoto, Yasuo and Saito, Yuta},
  booktitle={Proceedings of the 29th ACM SIGKDD Conference on Knowledge Discovery and Data Mining},
  pages={1154--1163},
  year={2023}
}

@article{sachdeva2023off,
  title={Off-Policy Evaluation for Large Action Spaces via Policy Convolution},
  author={Sachdeva, Noveen and Wang, Lequn and Liang, Dawen and Kallus, Nathan and McAuley, Julian},
  journal={arXiv preprint arXiv:2310.15433},
  year={2023}
}

@inproceedings{borisov2016neural,
  title={A neural click model for web search},
  author={Borisov, Alexey and Markov, Ilya and De Rijke, Maarten and Serdyukov, Pavel},
  booktitle={Proceedings of the 25th International Conference on World Wide Web},
  pages={531--541},
  year={2016}
}

@article{Dud_k_2014,
   title={Doubly Robust Policy Evaluation and Optimization},
   volume={29},
   ISSN={0883-4237},
   url={http://dx.doi.org/10.1214/14-STS500},
   DOI={10.1214/14-sts500},
   number={4},
   journal={Statistical Science},
   publisher={Institute of Mathematical Statistics},
   author={Dudík, Miroslav and Erhan, Dumitru and Langford, John and Li, Lihong},
   year={2014},
   month=nov }

@misc{liang2023local,
      title={Local Policy Improvement for Recommender Systems}, 
      author={Dawen Liang and Nikos Vlassis},
      year={2023},
      eprint={2212.11431},
      archivePrefix={arXiv},
      primaryClass={cs.LG}
}

@inproceedings{xu2012incorporating,
  title={Incorporating revisiting behaviors into click models},
  author={Xu, Danqing and Liu, Yiqun and Zhang, Min and Ma, Shaoping and Ru, Liyun},
  booktitle={Proceedings of the fifth ACM international conference on Web search and data mining},
  pages={303--312},
  year={2012}
}

@inproceedings{ai2018unbiased,
  title={Unbiased learning to rank with unbiased propensity estimation},
  author={Ai, Qingyao and Bi, Keping and Luo, Cheng and Guo, Jiafeng and Croft, W Bruce},
  booktitle={The 41st international ACM SIGIR conference on research \& development in information retrieval},
  pages={385--394},
  year={2018}
}

@inproceedings{wang2018position,
  title={Position bias estimation for unbiased learning to rank in personal search},
  author={Wang, Xuanhui and Golbandi, Nadav and Bendersky, Michael and Metzler, Donald and Najork, Marc},
  booktitle={Proceedings of the eleventh ACM international conference on web search and data mining},
  pages={610--618},
  year={2018}
}

@misc{zhang2023unified,
      title={Unified Off-Policy Learning to Rank: a Reinforcement Learning Perspective}, 
      author={Zeyu Zhang and Yi Su and Hui Yuan and Yiran Wu and Rishab Balasubramanian and Qingyun Wu and Huazheng Wang and Mengdi Wang},
      year={2023},
      eprint={2306.07528},
      archivePrefix={arXiv},
      primaryClass={cs.LG}
}

@inproceedings{liu2022practical,
  title={Practical counterfactual policy learning for top-k recommendations},
  author={Liu, Yaxu and Yen, Jui-Nan and Yuan, Bowen and Shi, Rundong and Yan, Peng and Lin, Chih-Jen},
  booktitle={Proceedings of the 28th ACM SIGKDD Conference on Knowledge Discovery and Data Mining},
  pages={1141--1151},
  year={2022}
}

@inproceedings{cief2024learning,
  title={Learning action embeddings for off-policy evaluation},
  author={Cief, Matej and Golebiowski, Jacek and Schmidt, Philipp and Abedjan, Ziawasch and Bekasov, Artur},
  booktitle={European Conference on Information Retrieval},
  pages={108--122},
  year={2024},
  organization={Springer}
}

@inproceedings{kiyohara2024off,
  title={Off-policy evaluation of slate bandit policies via optimizing abstraction},
  author={Kiyohara, Haruka and Nomura, Masahiro and Saito, Yuta},
  booktitle={Proceedings of the ACM on Web Conference 2024},
  pages={3150--3161},
  year={2024}
}

@inproceedings{
    tanaka2026offpolicy,
    title={Off-Policy Evaluation for Ranking Policies under Deterministic Logging Policies},
    author={Koichi Tanaka and Kazuki Kawamura and Takanori Muroi and Yusuke Narita and Yuki Sasamoto and Kei Tateno and Takuma Udagawa and Wei-Wei Du and Yuta Saito},
    booktitle={The Fourteenth International Conference on Learning Representations},
    year={2026},
    url={https://openreview.net/forum?id=0ZkWWxcHKV}
}

@inproceedings{tanaka2026off,
  title={Off-Policy Learning with Limited Supply},
  author={Tanaka, Koichi and Kishimoto, Ren and Kawagishi, Bushun and Narita, Yusuke and Yamamoto, Yasuo and Shimizu, Nobuyuki and Saito, Yuta},
  booktitle={Proceedings of the ACM Web Conference 2026},
  pages={5908--5919},
  year={2026}
}

@inproceedings{saito2025potec,
  title={POTEC: Off-policy contextual bandits for large action spaces via policy decomposition},
  author={Saito, Yuta and Yao, Jihan and Joachims, Thorsten},
  booktitle={International Conference on Learning Representations},
  volume={2025},
  pages={57640--57664},
  year={2025}
}

@inproceedings{saito2023off,
  title={Off-policy evaluation for large action spaces via conjunct effect modeling},
  author={Saito, Yuta and Ren, Qingyang and Joachims, Thorsten},
  booktitle={international conference on Machine learning},
  pages={29734--29759},
  year={2023},
  organization={PMLR}
}

\end{document}